%% file: article_CMSJ_IEEE_TASE_final.tex
\documentclass[letterpaper, 10pt, journal, twoside]{IEEEtran}

\usepackage{amsmath,amsfonts}
\usepackage{algorithmic}
\usepackage{algorithm}
\usepackage{array}
\usepackage[caption=false,font=normalsize,labelfont=sf,textfont=sf]{subfig}
\usepackage{textcomp}
\usepackage{stfloats}
\usepackage{url}
\usepackage{verbatim}
\usepackage{graphicx}
\usepackage{cite}
\usepackage{soul}
\usepackage{xcolor}
\usepackage{booktabs}
\usepackage{multirow}
\usepackage{fancyhdr}
\fancypagestyle{firstpage}{%
  \fancyhf{}
  
  \fancyhead[C]{\footnotesize This work has been submitted to the IEEE for possible publication. Copyright may be transferred without notice, after which this version may no longer be accessible.}
}
\usepackage{epsfig} 

\usepackage{bm}
\usepackage{lipsum}

\usepackage{tikz}
\usetikzlibrary{arrows.meta, calc, patterns}
\usetikzlibrary{shapes.geometric}

\newtheorem{property}{Property}
\newtheorem{remark}{Remark}
\newtheorem{assumption}{Assumption}
\newtheorem{proposition}{Proposition}
\newtheorem{problem}{Problem}
\newtheorem{corollary}{Corollary}

\begin{document}

\title{Lyapunov-Based PI-Like Control for Robust Trajectory Tracking of a Four-Wheel Independently Driven and Steered Robot: Design and Experimental Validation}

\author{Branimir Ćaran$^{1}$, Vladimir Milić$^{1}$, Marko Švaco$^{1}$ and Bojan Jerbić$^{1,2}$%
	\thanks{Manuscript received: Month, Day, Year; Revised Month, Day, Year; Accepted Month, Day, Year.}
	\thanks{This paper was recommended for publication by Editor Editor A. Name upon evaluation of the Associate Editor and Reviewers’ comments.
		This work was supported by Regional Centre of Excellence for Robotic Technology (CRTA), funded by the ERDF fund.}
	\thanks{$^{1}$Authors are with the Faculty of Mechanical Engineering and Naval Architecture, University of Zagreb, Zagreb HR-10000, Croatia
		{\tt\footnotesize branimir.caran@fsb.unizg.hr}}%
	\thanks{$^{2}$Bojan Jerbić is with Croatian Academy of Sciences and Arts, Zagreb 10000, Croatia.}%
	\thanks{Digital Object Identifier (DOI): see top of this page.}
}


\maketitle
\thispagestyle{firstpage}

\begin{abstract}
In this paper, a Lyapunov-based synthesis of a PI-like controller is proposed for robust trajectory tracking of an independently driven and steered four-wheel mobile robot. For the robot considered in this work, an explicit structurally verified mathematical model is used to enable systematic controller design with rigorous stability guarantees suitable for real time implementation. An augmented Lyapunov-based practical stability analysis is developed for the combined velocity-error and integral-error dynamics of the inner loop, yielding explicit bounds and sufficient conditions for practical stability and uniform ultimate boundedness of the combined velocity-error and integral-error state. The resulting control law retains a PI-like structure with model-based feedforward compensation, making it suitable for implementation on standard embedded platforms while improving robustness against configuration dependent residual dynamics and unmodelled effects. The effectiveness and robustness of the proposed design are demonstrated experimentally on a four-wheel independently steered and independently driven mobile robot platform, under both horizontal and vertical operating conditions and benchmarked against a PI controller and a sliding-mode controller.
\end{abstract}

\renewcommand{\abstractname}{Note to Practitioners}
\begin{abstract}
This work is motivated by the practical need to make four-wheel independent steer and four-wheel independent drive mobile robots track reference trajectories accurately and reliably under real operating conditions, where actuator imperfections, model uncertainty, and residual dynamics degrade performance. In industrial and field settings, practitioners often favor PI controllers because they are simple to tune and run on standard embedded hardware, but a plain PI structure provides no formal stability guarantee and can lose accuracy when the robot's effective dynamics change with its configuration. More advanced control systems such as sliding-mode control improve robustness but tend to introduce chattering and demand greater tuning effort and computation. The method proposed here retains the familiar PI-like structure that engineers are comfortable deploying, extends it with a model-based feedforward term, and derives the control gains from a Lyapunov-based analysis. This gives the designer explicit conditions under which the tracking error is guaranteed to stay bounded and converge close to zero, rather than relying on trial and error tuning alone. To apply the method, the user needs a structurally verified dynamic model of the platform; the paper supplies such a model and shows how the stability conditions translate into concrete gain bounds. The resulting control law is computationally light and was validated experimentally on a real robot platform in both horizontal and vertical operating conditions, outperforming a conventional PI controller and a sliding-mode controller.
\end{abstract}

\renewcommand{\abstractname}{Abstract}


\section{Introduction}

\subsection{Motivation}

The motivation for considering a four-wheeled mobile robot with independent steering and driving arises from the need for accurate motion control in complex scenarios and dynamically changing environments. This robot configuration offers high manoeuvrability and actuation redundancy, and these capabilities have motivated substantial research on this class of robots, leading to a variety of control strategies; see, for example, \cite{LSLLLL25, THMWC25, SABGDWKC25, ZHWMLX21, JWXMZZW20, LLLLLC16} and the references therein. However, the same configuration also introduces a nontrivial control problem: robust trajectory tracking must be achieved in the presence of configuration-dependent residual dynamics and unmodelled effects, such as friction variations, gravity-related effects during vertical surface operation, and adhesion/contact phenomena, while consistently accounting for the underlying dynamics, i.e., real forces, torques, and constraints, within a unified control framework.

This motivates a control design that combines model-based compensation with a simple feedback structure suitable for real-time implementation. In this work, a Lyapunov-based PI-like controller is developed for the considered robot platform. The proposed formulation uses an explicit structurally verified model to derive residual error dynamics and to construct an augmented Lyapunov function for the closed-loop system. This yields explicit bounds and sufficient conditions for practical stability and uniform ultimate boundedness of the combined velocity-error and integral-error state. In this way, robustness with respect to configuration-dependent residual terms and bounded unmodelled effects is addressed while retaining a controller structure compatible with standard embedded implementation.

\subsection{Related Work}

In the control of wheeled mobile robots, passivity provides a physically meaningful framework for analysing energy exchange, stability, and robustness. In particular, passivity and output strict passivity have often been used to obtain energy-based stability properties and, under appropriate assumptions, finite $\mathcal{L}_2$-gain estimates relating disturbance inputs to tracking or regulation errors. Such results are especially useful when the closed-loop dynamics can be written in a dissipative form with respect to a well-defined storage function and supply rate. Consequently, passivity and $\mathcal{L}_2$-based robustness analysis have been extensively explored in the mobile-robot literature as tools for formalizing robustness of feedback control systems.

The authors in \cite{MRTC12} combined passivity, dynamic compensation, and visual servo control with formal $\mathcal{L}_2$-gain robustness analysis and experimental validation on a unicycle-type robot, under the assumptions of good calibration and near-perfect knowledge of the moving object velocity. Similarly, the approach in \cite{ARSTC15} provides a framework for visual servo control of mobile manipulators with stability and $\mathcal{L}_2$ performance analyses. By introducing passive configuration decomposition, the authors in \cite{LL17} proposed a framework for stabilizing a class of nonholonomic mechanical systems that exploits their Lagrangian structure and passivity, under assumptions such as a product structure of the configuration space, zero nonholonomic connection, and a local description. A passivity-based port-Hamiltonian approach that, through a dynamic extension, allows a group of mobile robots to follow a desired trajectory and formation using only position measurements, without velocity measurements, was proposed in \cite{LBSM23}. In that work, all robots are assumed to have the same constant desired orientation and velocity, while the communication graph is restricted to a certain class of motion and connection scenarios. In \cite{CW25}, passivity and barrier functions are combined with quadratic programming synthesis for safe maneuvering control of mechanical systems under constraints, assuming an exact system model and nonsingularity of the control input matrix.

Besides passivity-based methods, adaptive and sliding-mode approaches have also been widely used to improve robustness of wheeled mobile robots. The robust adaptive sliding controller proposed in \cite{WWXLZJ24} attenuates abrupt disturbances within allowable bounds and ensures finite-time convergence, although it requires complex parameter tuning and provides limited treatment of noise sensitivity and unmodelled nonlinear effects. Furthermore, the robust adaptive fault-tolerant control strategy based on nonsingular terminal sliding mode control and barrier functions proposed in \cite{WWLXZX25} addresses multiple actuator faults and input saturation with finite-time convergence guarantees, but relies on a specific fault model.

Related principles and methodologies developed for autonomous road vehicles can also be adapted to four-wheeled mobile robots. For example, robust path tracking, port-Hamiltonian/passivity-based observation, and yaw-stability-oriented path tracking have been investigated in \cite{WZWZL25, MHSW23, MCWXW22}. Although these methods provide useful insights for four-wheel independently steered and driven platforms, they are usually developed in the context of autonomous vehicles and are not directly tailored to the residual dynamics, steering-drive redundancy, and contact-dependent effects encountered in compact mobile robots operating on horizontal and vertical surfaces.

Despite this active body of research, robust control of independently steered and driven four-wheeled mobile robots with a simple implementation-oriented PI-like structure, explicit residual-dynamics bounds, and comprehensive experimental validation remains insufficiently addressed. In contrast to works that formulate robustness primarily through passivity or $\mathcal{L}_2$-gain estimates, the present work uses passivity-inspired structural properties of the robot dynamics as a basis for a Lyapunov-based practical stability analysis. This leads to explicit sufficient conditions for practical stability and uniform ultimate boundedness of the combined velocity-error and integral-error state in the presence of configuration-dependent residual terms and bounded unmodelled effects.

\subsection{Contributions}

This work develops a Lyapunov-based PI-like control framework for robust trajectory tracking of a four-wheel mobile robot with independent steering and driving under a symmetric steering configuration. The proposed approach uses a reduced velocity-space model with bounded unmodelled effects represented as memoryless nonlinearities, and combines model-based feedforward compensation with proportional-integral feedback to establish practical stability and uniform ultimate boundedness of the combined inner-loop velocity-error and integral-error state.

The main contributions of this work are summarized as follows:
\begin{itemize}
\item A reduced velocity-space dynamic model of the considered mobile robot is developed, with explicitly derived diagonal inertia, Coriolis, and input maps. The core structural properties of this model, including positive definiteness, the energy-balance identity, input regularity, and uniform spectral bounds, are established. Based on these properties, explicit residual-dynamics, Lipschitz, kinematic mismatch, and uncertainty bounds are derived, providing a constructive foundation for robust trajectory-tracking control design.

\item A Lyapunov-based PI-like control law with model-based feedforward compensation is synthesized in the reduced velocity space. A Lyapunov-based practical stability analysis is developed for the closed-loop error dynamics, yielding explicit sufficient gain conditions for practical stability and uniform ultimate boundedness of the combined velocity-error and integral-error state. The resulting tuning conditions are expressed in terms of constants extracted from the robot dynamics and the residual bounds.

\item The proposed control strategy is experimentally validated on a four-wheel independently driven and steered mobile robot platform on both horizontal and vertical surfaces, and benchmarked against a PI controller and a sliding-mode controller. The experiments demonstrate robust trajectory tracking in the presence of external disturbances, unmodelled memoryless nonlinearities, gravity-related effects, adhesion/contact phenomena, and varying contact conditions, confirming the practical robustness and applicability of the proposed control framework. The comparison shows that the proposed controller achieves lower tracking error and reduced actuation effort than both baseline controllers under these conditions.
\end{itemize}

The paper is organized as follows. A detailed mathematical model of the considered mobile robot is given in Section \ref{sec_SystemDescription}, where the kinematic model, reduced dynamic model, and its structural properties are derived. The proposed PI-like control strategy, residual error dynamics, and Lyapunov-based practical stability analysis are presented in Section \ref{sec_ControllerDesign}. Section \ref{sec_Results} describes the experimental setup and reports trajectory-tracking results on horizontal and vertical surfaces. Finally, concluding remarks are provided in Section \ref{sec_Conclusion}.

\section{Mathematical model and properties}\label{sec_SystemDescription}

\subsection{Kinematics}

The kinematics of four-wheel mobile robots with independent steering and driving, corresponding to the robot structure considered in this work, have been derived in \cite{LSLLLL25, THMWC25, LWLLHYRZ22, LL15}. Following \cite{CMSJ24}, several simplifying assumptions are introduced here to obtain a reduced kinematic model suitable for the subsequent control design. A schematic representation of this robot is shown in Fig. \ref{fig1_robot_schematic}.

The assumptions introduced form the so-called symmetric model where the velocities of the wheels are equal and the front and rear wheels are steered parallel to each other and at the same steering speed. For the sake of completeness, these kinematic equations of the robot are also given here in the following compact form:
\begin{equation}\label{eq_kinem_compact}
\mathbf{\dot{q}} = \mathbf{J}(\mathbf{q}) \mathbf{v},
\end{equation}
where $\mathbf{q} \in \mathbb{R}^{6}$ is the generalized coordinates (state) vector, $\mathbf{v} \in \mathbb{R}^{3}$ is the velocity input vector and $\mathbf{J}(\mathbf{q}) \in \mathbb{R}^{6 \times 3}$ is the kinematic Jacobian matrix which are given as follows:
\begin{align}
\mathbf{q} &= \begin{bmatrix} x & y & \theta & \varphi & \delta_f & \delta_r \end{bmatrix}^{\text{T}}, \label{eq1a_kinem_compact}\\
\mathbf{v} &= \begin{bmatrix} v_w & \omega_f & \omega_r \end{bmatrix}^{\text{T}}, \label{eq1b_kinem_compact} \\
\mathbf{J}(\mathbf{q}) &=
\begin{bmatrix}
\frac{1}{2} (\cos(\delta_f + \theta) + \cos(\delta_r + \theta)) & 0 & 0 \\
\frac{1}{2} (\sin(\delta_f + \theta) + \sin(\delta_r + \theta)) & 0 & 0 \\
A (\sin(\delta_f) - \sin(\delta_r)) & 0 & 0 \\
r^{-1} & 0 & 0 \\
0 & 1 & 0 \\
0 & 0 & 1
\end{bmatrix}. \label{eq1c_kinem_compact}
\end{align}

In \eqref{eq_kinem_compact}-\eqref{eq1c_kinem_compact} the coordinates $x$ and $y$ represent the reference point of the robot in the Cartesian frame, $\theta$ is the orientation angle which is expressed in relation to the positive $x$-axis, $\varphi$ is the wheel spin angle, $v_w$ is the wheel tangential (peripheral) linear velocity, $\delta_f$, $\omega_f$, $\delta_r$ and $\omega_r$ are the steering angles and steering velocities of the front and rear wheels, respectively and $r$ is the wheel radius. Furthermore,
\begin{equation}\label{eq2c_kinem_compact}
A = \frac{a}{2(a^2 + b^2)},
\end{equation} 
where $a$ and $b$ denote the longitudinal and lateral distances from the robot centroid to each wheel (i.e., half of the wheelbase and half of the track width, see Fig. \ref{fig1_robot_schematic}).

Throughout the paper we use the kinematic relation $\mathbf{\dot{q}} = \mathbf{J}(\mathbf{q}) \mathbf{v}$. Hence any term that depends on $\mathbf{\dot{q}}$ can be equivalently expressed as a function of $(\mathbf{q}, \mathbf{v})$ by evaluating it at $\mathbf{\dot{q}} = \mathbf{J}(\mathbf{q}) \mathbf{v}$.

\begin{remark}
It should be noted that the kinematic model considered here assumes rolling without lateral slip. Under this assumption, the considered robot platform is a nonholonomic system, since each wheel enforces a non-integrable zero lateral velocity constraint and the wheel steering angles are state variables. Therefore, lateral translation is achieved only indirectly after wheel reorientation, not instantaneously. If lateral slip is allowed through a frictional contact model, the robot may be modeled as a holonomic system.
\end{remark}

\begin{figure}[!t]
\centering
\scalebox{0.6}{\input{Figures/schema_robot_horiz_new_ver2.tex}}
\caption{Mobile robot schematics.} \label{fig1_robot_schematic}
\end{figure}

\subsection{Dynamics}

It is well-known that the dynamics equations can be obtained by modelling the system using the Euler-Lagrange equations which are based on setting the total energy of the system. The total kinetic energy of the simplified symmetric robot system considered in this work is as follows:
\begin{equation}\label{eq_kinetic_energy}
\begin{aligned}
T &= \frac{1}{2} m(\dot{x}^2 + \dot{y}^2) + \frac{1}{2} I \dot{\theta}^2 + \frac{1}{2} (4 I_{\varphi}) \dot{\varphi}^2 \\
&+ \frac{1}{2} (2 I_\delta) \dot{\delta}_f^2 + \frac{1}{2} (2 I_\delta) \dot{\delta}_r^2,
\end{aligned}
\end{equation}
where $m = m_b + 4m_w$ is the total mass of the robot while $m_b$ and $m_w$ are the masses of the robot's body and each wheel, respectively, $I = I_\theta + 4m_w(a^2+b^2)$ is the yaw moment of inertia of the chassis about the $z$-axis through the centre of the mass, $I_{\theta}$ is central yaw inertia of the chassis about the same axis, $I_{\varphi}$ is the spin moment of inertia of one wheel about its axle, $I_{\delta}$ is the steering-axis moment of inertia of one wheel about its steering axis. Note that in \eqref{eq_kinetic_energy} the factors $4I_{\varphi}$ and $2I_{\delta}$ appear because $\varphi$ is shared by all four wheels, while $\delta_f$ and $\delta_r$ are each shared by two steerable wheels (front and rear), respectively.

Furthermore, by choosing the vector of control torques of the system as $\bm{\tau} = \begin{bmatrix} \tau_w & \tau_f & \tau_r\end{bmatrix}^{\text{T}} \in \mathbb{R}^{3}$, where $\tau_w$ is the average wheel spin torque, $\tau_f$ is the average front steering torque, and $\tau_r$ the average rear steering torque $\tau_w = \frac{1}{4} \sum_{i=1}^{4} \tau_{d i}$, $\tau_f = \frac{1}{2} \left( \tau_{s 1} + \tau_{s 4} \right)$, $\tau_r = \frac{1}{2} \left( \tau_{s 2} + \tau_{s 3} \right)$ where $\tau_{d i}$ is the drive torque of the $i$-th wheel, and $\tau_{s i}$ is the torque of the $i$-th wheel's steering actuator. Then based on the Euler-Lagrange formalism from the expression for kinetic energy \eqref{eq_kinetic_energy}, after certain calculations and algebraic manipulations, the dynamics of the robot can be written in the following compact form of a matrix differential equation
\begin{equation}
\begin{aligned}
\mathbf{M}(\mathbf{q}) \mathbf{\ddot{q}} + \mathbf{C}(\mathbf{q}, \mathbf{\dot{q}}) \mathbf{\dot{q}} + \mathbf{f}(\mathbf{q}, \mathbf{\dot{q}}) &= \mathbf{B}(\mathbf{q}) \bm{\tau} - \mathbf{A}^{\text{T}}(\mathbf{q}) \bm{\lambda},\\ 
\mathbf{A}(\mathbf{q}) \mathbf{\dot{q}} &= \mathbf{0}, \label{eq_robot_dynamics}
\end{aligned}
\end{equation}
where $\mathbf{M}(\mathbf{q}) \in \mathbb{R}^{6 \times 6}$ is the inertia matrix, $\mathbf{C}(\mathbf{q}, \mathbf{\dot{q}}) \in \mathbb{R}^{6 \times 6}$ is the Coriolis and centripetal matrix and $\mathbf{B}(\mathbf{q}) \in \mathbb{R}^{6 \times 3}$ is the input matrix. The elements of these matrices are as follows:
\begin{align}
\mathbf{M}(\mathbf{q}) &= \mathrm{diag}\big(m,\, m,\, I,\, 4I_{\varphi},\, 2I_\delta,\, 2I_\delta\big),\\
\mathbf{C}(\mathbf{q}, \mathbf{\dot{q}}) &= \mathbf{0},\\
\mathbf{B}(\mathbf{q}) &=
\begin{bmatrix}
\frac{2}{r} \big(\cos(\delta_f + \theta) + \cos(\delta_r + \theta)\big) & 0 & 0\\
\frac{2}{r} \big(\sin(\delta_f + \theta) + \sin(\delta_r + \theta)\big) & 0 & 0\\
\frac{2a}{r} \big(\sin(\delta_f) - \sin(\delta_r)\big) & 0 & 0\\
4 & 0 & 0\\
0 & 2 & 0\\
0 & 0 & 2
\end{bmatrix}.
\end{align}

In \eqref{eq_robot_dynamics} the constraint matrix $\mathbf{A}(\mathbf{q}) \in \mathbb{R}^{3 \times 6}$ is given as follows
\begin{equation}
\mathbf{A}(\mathbf{q}) =
\begin{bmatrix}
1 & 0 & 0 & -r a_{2} & 0 & 0\\
0 & 1 & 0 & -r a_{1} & 0 & 0\\
0 & 0 & 1 & -r a_{3} & 0 & 0
\end{bmatrix},
\end{equation}
with $a_{1} = \frac{1}{2} \left( \sin(\delta_f + \theta) + \sin(\delta_r + \theta) \right)$, $a_{2} = \frac{1}{2} \left( \cos(\delta_f + \theta) + \cos(\delta_r+\theta) \right)$, $a_{3} = A \left( \sin(\delta_f) - \sin(\delta_r) \right)$, while $\bm{\lambda} \in \mathbb{R}^{3}$ is the vector of Lagrange multipliers, i.e. vector of constraint reaction forces.

\begin{remark}
Since mathematical simplifications are introduced in this work to obtain the kinematic model \eqref{eq_kinem_compact}, the constraint matrix $\mathbf{A}(\mathbf{q})$ is constructed as a full-rank left-nullspace complement of the derived Jacobian $\mathbf{J}(\mathbf{q})$. Hence, it satisfies $\mathbf{A}(\mathbf{q}) \mathbf{J}(\mathbf{q}) = \mathbf{0}$, $\operatorname{rank}\mathbf{A}(\mathbf{q}) = 3$, and therefore $\ker \mathbf{A}(\mathbf{q}) = \operatorname{im}\mathbf{J}(\mathbf{q})$. Consequently, the admissible generalized velocities satisfying $\mathbf{A}(\mathbf{q}) \mathbf{\dot{q}}=\mathbf{0}$ are exactly those that can be represented in the reduced form $\mathbf{\dot{q}} = \mathbf{J}(\mathbf{q}) \mathbf{v}$. In this sense, the constraint matrix $\mathbf{A}(\mathbf{q})$, apart from being related to the rolling-without-slipping assumptions, should be interpreted as a mathematical constraint induced by the reduced symmetric kinematic model. This construction enables the constrained Euler-Lagrange model to be projected into the reduced control-oriented form (see, e.g., \cite{CLHKBKT05, FL98, SYK94}) used in the subsequent analysis.
\end{remark}

The vector $\mathbf{f}(\mathbf{q}, \mathbf{\dot{q}}) \in \mathbb{R}^{6}$ collects unknown effects and the following assumption is introduced.
\begin{assumption}\label{assum_f}
Let the vector $\mathbf{f}(\mathbf{q}, \mathbf{\dot{q}})$ in \eqref{eq_robot_dynamics} represents bounded uncertainties in the robot dynamics and can be in the form of external disturbances and/or unmodelled static nonlinearities. Here, “static” refers to memoryless nonlinearities, i.e., algebraic uncertainty terms that depend instantaneously on the current states and velocities and do not introduce additional internal dynamics.

Let the uncertainty act on all coordinates i.e. 
\begin{equation}
\mathbf{f}(\mathbf{q}, \mathbf{\dot{q}}) = \begin{bmatrix} f_{x} & f_{y} & f_{\theta} & f_{\varphi} & f_{\delta_f} & f_{\delta_r} \end{bmatrix}^{\text{T}}. 
\end{equation}
Assume there exist nonnegative functions $g_k(\mathbf{q})$, $h_k(\mathbf{q})$ such that
\begin{equation}
|f_{k}(\mathbf{q}, \mathbf{\dot{q}})| \leq g_{k}(\mathbf{q}) + h_{k}(\mathbf{q}) |\dot{k}|,
\end{equation}
for $k \in \{x, y, \theta, \varphi, \delta_f, \delta_r\}$. Note that each component $f_k$ depends on $\mathbf{\dot{q}}$ only through its corresponding scalar velocity $\dot{k}$. Let uniform bounds be defined on a compact set $\mathcal{X} \subset \mathbb{R}^{6}$ as follows
\begin{equation}\label{eq_sup_ci_di}
c_k = \sup_{\mathbf{q} \in \mathcal{X}} g_k(\mathbf{q}),\quad d_k = \sup_{\mathbf{q} \in \mathcal{X}} h_k(\mathbf{q}).
\end{equation}
Then
\begin{equation}\label{eq_assum_fd}
\|\mathbf{f}\| \leq \|\mathbf{c}\| + \|\mathbf{d}\| \left\| \mathbf{\dot{q}} \right\|,
\end{equation}
where $\mathbf{c} = \left[ c_{x}\; c_{y}\; c_{\theta}\; c_{\varphi}\; c_{\delta_f}\; c_{\delta_r} \right]^{\text{T}}$, $\mathbf{d} = \left[d_{x}\; d_{y}\; d_{\theta}\; d_{\varphi}\; d_{\delta_f}\; d_{\delta_r}\right]^{\text{T}}$.

Furthermore, $\mathbf{f}$ is assumed to be locally Lipschitz. Let $\mathbf{f} : \mathbb{R}^{6} \times \mathbb{R}^{6} \to \mathbb{R}^{6}$ be $C^{1}$ on an open set containing $\mathcal{X} \times \mathcal{D}$, where $\mathcal{X} \subset \mathbb{R}^{6}$ and $\mathcal{D} \subset \mathbb{R}^{6}$ are compact. Then, with the induced spectral norm,
\begin{equation}\label{eq_assum_Lf_a}
\begin{aligned}
L_{f_1} &= \sup_{\mathbf{q}, \mathbf{\dot{q}} \in \mathcal{X} \times \mathcal{D}} \left\| \frac{\partial \mathbf{f}}{\partial \mathbf{q}} \right\| < \infty,\\ 
L_{f_2} &= \sup_{\mathbf{q}, \mathbf{\dot{q}} \in \mathcal{X} \times \mathcal{D}} \left\| \frac{\partial \mathbf{f}}{\partial \mathbf{\dot{q}}} \right\| < \infty,
\end{aligned}
\end{equation}
and hence $\left\| \partial \mathbf{f}/\partial \mathbf{q} \right\| \leq L_{f_1}$, $\left\| \partial \mathbf{f}/\partial \mathbf{\dot{q}} \right\| \leq L_{f_2}$, $\forall \mathbf{q}, \mathbf{\dot{q}} \in \mathcal{X} \times \mathcal{D}$.

For instance, unmodelled friction effects, gravity-related effects (during vertical surface operation), and adhesion/contact phenomena may be lumped into $\mathbf{f}(\mathbf{q}, \mathbf{\dot{q}})$ and treated as bounded disturbances on the operating set.
\end{assumption}

Next, the procedure as in \cite{WWXLZJ24, WWLXZX25, LBY20, CLHKBKT05, FL98, SYK94, LRLLW25, CWYNW21} is carried out to transform the robot model into a control-oriented form. Substituting \eqref{eq_kinem_compact} and its derivative $\mathbf{\ddot{q}} = \mathbf{\dot{J}}(\mathbf{q}) \mathbf{v} + \mathbf{J}(\mathbf{q}) \mathbf{\dot{v}}$ into \eqref{eq_robot_dynamics}, then multiplying by $\mathbf{J}^{\text{T}}(\mathbf{q})$ and taking into account that $\mathbf{J}^{\text{T}}(\mathbf{q}) \mathbf{A}^{\text{T}}(\mathbf{q}) = \mathbf{0}$ finally yields
\begin{equation}\label{eq_dynamics_in_v}
\mathbf{\tilde{M}}(\mathbf{q}) \mathbf{\dot{v}} + \mathbf{\tilde{C}}(\mathbf{q}, \mathbf{\dot{q}}) \mathbf{v} = \mathbf{\tilde{B}}(\mathbf{q}) \bm{\tau} - \mathbf{\tilde{f}}(\mathbf{q}, \mathbf{v}),
\end{equation}
where $\mathbf{\tilde{M}}(\mathbf{q}) \in \mathbb{R}^{3 \times 3}$ is the inertia matrix, $\mathbf{\tilde{C}}(\mathbf{q}, \mathbf{\dot{q}}) \in \mathbb{R}^{3 \times 3}$ is the Coriolis and centripetal matrix, $\mathbf{\tilde{B}}(\mathbf{q}) \in \mathbb{R}^{3 \times 3}$ is the input matrix and $\mathbf{\tilde{f}} \in \mathbb{R}^3$ is the vector of bounded uncertainties. These matrices and vectors are obtained as follows:

\begin{equation}\label{eq_tildeM}
\begin{aligned}
\mathbf{\tilde{M}}(\mathbf{q}) &= \mathbf{J}^{\text{T}}(\mathbf{q})\, \mathbf{M}(\mathbf{q})\, \mathbf{J}(\mathbf{q}) =  \mathrm{diag}\big(\tilde{M}_{11},\, 2I_\delta,\, 2I_\delta \big),\\
\tilde{M}_{11} &= \frac{4 I_{\varphi}}{r^2} + \frac{m}{2}\big(1+\cos(\delta_f - \delta_r)\big) \\
&+ I A^2 \big(\sin(\delta_f) - \sin(\delta_r)\big)^{2},
\end{aligned}
\end{equation}

\begin{equation}\label{eq_tildeC}
\begin{split}
\mathbf{\tilde{C}}(\mathbf{q}, \mathbf{\dot{q}}) &= \mathbf{J}^{\text{T}}(\mathbf{q})\, \mathbf{M}(\mathbf{q})\, \mathbf{\dot{J}}(\mathbf{q}) = \mathrm{diag}(\tilde{C}_{11},\, 0,\, 0),\\
\tilde{C}_{11}
&= I A^2 \big(\sin(\delta_f) - \sin(\delta_r)\big) \\
&\cdot \big(\cos(\delta_f) \dot{\delta}_f - \cos(\delta_r) \dot{\delta}_r\big) \\
&- \frac{m}{4}\sin(\delta_f - \delta_r) \big(\dot{\delta}_f - \dot{\delta}_r\big),
\end{split}
\end{equation}

\begin{equation}\label{eq_tildeB}
\begin{split}
\mathbf{\tilde{B}}(\mathbf{q}) &= \mathbf{J}^{\text{T}}(\mathbf{q})\, \mathbf{B}(\mathbf{q}) = \mathrm{diag}\big(\tilde{B}_{11},\, 2,\, 2\big),\\
\tilde{B}_{11} &= \frac{1}{r} \Big[2\big(1 + \cos(\delta_f - \delta_r)\big) \\
&+ \frac{a^2}{a^2+b^2} \big(\sin(\delta_f) - \sin(\delta_r)\big)^{2} + 4\Big],
\end{split}
\end{equation}

\begin{equation}\label{eq_tildefd}
\begin{split}
\mathbf{\tilde{f}}(\mathbf{q}, \mathbf{v}) &= \mathbf{J}^{\text{T}}(\mathbf{q}) \mathbf{f}(\mathbf{q}, \mathbf{J}(\mathbf{q})\mathbf{v}) = \begin{bmatrix} \tilde{f} & f_{\delta_f} & f_{\delta_r} \end{bmatrix}^{\text{T}},\\
\tilde{f} &= \frac{1}{2} (\cos(\delta_f + \theta) + \cos(\delta_r + \theta)) f_{x} \\
&+ \frac{1}{2} (\sin(\delta_f + \theta) + \sin(\delta_r + \theta)) f_{y} \\
&+ A (\sin(\delta_f) - \sin(\delta_r)) f_{\theta} + r^{-1} f_{\varphi}.
\end{split}
\end{equation}

System \eqref{eq_dynamics_in_v} with \eqref{eq_tildeM}-\eqref{eq_tildefd} possesses several structural and boundedness properties that are essential for the subsequent Lyapunov-based practical stability analysis. In particular, these properties provide uniform inertia bounds, the energy-balance identity of the reduced mechanical model, input regularity, and explicit estimates of the Coriolis and uncertainty terms. These results form the basis for constructing an augmented Lyapunov function that yields a strict practical stability estimate and for deriving sufficient gain conditions under which the proposed controller guarantees practical stability and uniform ultimate boundedness of the combined velocity-error and integral-error state in the presence of configuration-dependent residual dynamics and bounded unmodelled effects.

\begin{property}\label{property_1}
The matrix $\mathbf{\tilde{M}}(\mathbf{q})$ is symmetric positive-definite for all configurations of $\mathbf{q}$.
\end{property}

\begin{IEEEproof}
Obviously the matrix $\mathbf{\tilde{M}}(\mathbf{q})$ is diagonal with strictly positive diagonal entries. Indeed, $2I_\delta > 0$, and $\tilde{M}_{11} \geq \frac{4I_{\varphi}}{r^2} > 0$. Hence $\mathbf{z}^{\text{T}} \mathbf{\tilde{M}}(\mathbf{q}) \mathbf{z} = \tilde{M}_{11} z_1^2 + 2I_\delta(z_2^2 + z_3^2) > 0$ for any $\mathbf{z} \neq \mathbf{0}$.
\end{IEEEproof}

\begin{property}\label{property_2}
There exist positive constants $a_1$ and $a_2$ such that
\begin{equation}
a_1 \|\mathbf{z}\|^2 \leq \mathbf{z}^{\text{T}} \mathbf{\tilde{M}}(\mathbf{q}) \mathbf{z} \leq a_2 \|\mathbf{z}\|^2,\quad \forall \mathbf{z}\in\mathbb{R}^3.
\end{equation}
\end{property}

\begin{IEEEproof}
Property \ref{property_1} implies
\begin{equation}
\begin{aligned}
a_1(\mathbf{q}) &= \lambda_{\min}\left\{ \mathbf{\tilde{M}}(\mathbf{q}) \right\} = \min \{\tilde{M}_{11}(\mathbf{q}),\, 2I_\delta\},\\ 
a_2(\mathbf{q}) &= \lambda_{\max}\left\{ \mathbf{\tilde{M}}(\mathbf{q}) \right\} = \max \{\tilde{M}_{11}(\mathbf{q}),\, 2I_\delta\}.
\end{aligned}
\end{equation}
The quantities $\lambda_{\min}\{ \cdot \}$ and $\lambda_{\max}\{ \cdot \}$ denote, respectively, the smallest and largest eigenvalues of the symmetric positive-definite matrix.

Using $-1 \leq \cos(\delta_f - \delta_r) \leq 1$ and $0 \leq (\sin(\delta_f) - \sin(\delta_r))^2 \leq 4$ it follows
\begin{equation}
\frac{4 I_{\varphi}}{r^2} \leq \tilde{M}_{11}(\mathbf{q}) \leq \frac{4 I_{\varphi}}{r^2} + m + 4 I A^2.
\end{equation}
Therefore uniform (i.e. configuration independent) bounds may be chosen as
\begin{equation}
\begin{aligned}
a_1 &= \min \Big\{\frac{4 I_{\varphi}}{r^2},\, 2I_\delta\Big\},\\ 
a_2 &= \max \Big\{\frac{4 I_{\varphi}}{r^2} + m + 4 I A^2,\, 2I_\delta\Big\}.
\end{aligned}
\end{equation}
\end{IEEEproof}

\begin{property}\label{property_3}
The matrix $\mathbf{\dot{\tilde{M}}}(\mathbf{q}) - 2\mathbf{\tilde{C}}(\mathbf{q}, \mathbf{\dot{q}})$ is skew-symmetric. 
\end{property}

\begin{IEEEproof}
Note that in \eqref{eq_tildeM} only $\tilde{M}_{11}$ is time dependent. Differentiating $\tilde{M}_{11}$ gives
\begin{equation}
\begin{aligned}
\dot{\tilde{M}}_{11} &= -\frac{m}{2} \sin(\delta_f - \delta_r) \left( \dot{\delta}_f - \dot{\delta}_r \right) \\
&+ 2I A^2 \left( \sin(\delta_f) - \sin(\delta_r) \right) \\
&\cdot \left( \cos(\delta_f) \dot{\delta}_f - \cos(\delta_r) \dot{\delta}_r \right).
\end{aligned}
\end{equation}
Comparing $\dot{\tilde{M}}_{11}$ with $\tilde{C}_{11}$ from \eqref{eq_tildeC} it is easy to determine that $\dot{\tilde{M}}_{11} = 2\tilde{C}_{11}$. Hence $\mathbf{\dot{\tilde{M}}}(\mathbf{q}) - 2\mathbf{\tilde{C}}(\mathbf{q}, \mathbf{\dot{q}}) = \mathrm{diag}(0,\, 0,\, 0) = \mathbf{0}$, which is trivially skew-symmetric.
\end{IEEEproof}

\begin{remark}\label{remark_from_prop1_prop3}
Note that Property \ref{property_1} and Property \ref{property_3} imply the energy-balance identity $\mathbf{\dot{\tilde{M}}}(\mathbf{q}) = \mathbf{\tilde{C}}(\mathbf{q}, \mathbf{\dot{q}}) + \mathbf{\tilde{C}}(\mathbf{q}, \mathbf{\dot{q}})^{\text{T}}$ and hence $\mathbf{\tilde{C}}(\mathbf{q}, \mathbf{\dot{q}}) = \frac{1}{2} \mathbf{\dot{\tilde{M}}}(\mathbf{q})$.
\end{remark}

\begin{property}\label{property_4}
There exists positive constant $b_c$ such that for all admissible $\mathbf{v} = \begin{bmatrix} v_w & \omega_f & \omega_r \end{bmatrix}^{\text{T}} = \begin{bmatrix} v_w & \dot{\delta}_f & \dot{\delta}_r \end{bmatrix}^{\text{T}}$,
\begin{equation}\label{eq_Cor_bound}
\| \mathbf{\tilde{C}}(\mathbf{q}, \mathbf{\dot{q}}) \mathbf{v} \| \leq b_c \| \mathbf{v} \|^2.
\end{equation}

Moreover, on the compact operating set satisfying $\delta_f, \delta_r \in D_\delta$ and $|\dot{\delta}_f|, |\dot{\delta}_r| \leq \dot{\delta}_{\max}$, the matrix $\mathbf{\tilde{C}}(\mathbf{q}, \mathbf{\dot{q}})$ is uniformly bounded as
\begin{equation}
\|\mathbf{\tilde{C}}(\mathbf{q}, \mathbf{\dot{q}})\| \leq c_c. \label{eq:Cbar_bound}
\end{equation}
\end{property}

\begin{IEEEproof}
Since $ \mathbf{\tilde{C}}(\mathbf{q}, \mathbf{\dot{q}}) = \mathrm{diag}(\tilde{C}_{11},\, 0,\, 0)$ one has $\| \mathbf{\tilde{C}}(\mathbf{q}, \mathbf{\dot{q}}) \mathbf{v} \|=|\tilde{C}_{11}|\, |v_w|$ and $\|\mathbf{\tilde{C}}(\mathbf{q}, \mathbf{\dot{q}})\| = |\tilde C_{11}|$. From closed form of $\tilde{C}_{11}$ in \eqref{eq_tildeC} and taking into account $ |\sin(\cdot)| \leq 1$ and $|\cos(\cdot)| \leq 1 $ it follows that
\begin{equation}
\begin{aligned}
|\tilde{C}_{11}| &\leq  2 I A^2 \big( |\dot{\delta}_f| + |\dot{\delta}_r| \big) + \frac{m}{4} \big( |\dot\delta_f|+|\dot\delta_r| \big) \\
&= \Big( 2 I A^2 + \frac{m}{4} \Big) \big( |\dot{\delta}_f| + |\dot{\delta}_r| \big).
\end{aligned}
\end{equation}
Thus
\begin{equation}
\begin{aligned}
\| \mathbf{\tilde{C}}(\mathbf{q}, \mathbf{\dot{q}}) \mathbf{v} \| &\leq \Big(2 I A^2 + \frac{m}{4}\Big) |v_w| \big(|\dot{\delta}_f| + |\dot{\delta}_r|\big)\\ 
&\leq \Big(2 I A^2 + \frac{m}{4}\Big) \left( \frac{1}{2} v_w^2 + \dot{\delta}_f^{2} + \dot{\delta}_r^{2} \right)\\ 
&\leq b_c \|\mathbf{v}\|^2,
\end{aligned}
\end{equation}
where Young's inequality and $\|\mathbf{v}\|^2 = v_w^2 + \dot{\delta}_f^{2} + \dot{\delta}_r^{2}$ are used. The explicit obvious choice is
\begin{equation}\label{eq_Cor_bound_b_c}
b_c = 2 I A^2 + \frac{m}{4}.
\end{equation}

Furthermore, from the preceding estimate, using $|\dot{\delta}_f|, |\dot{\delta}_r| \leq \dot{\delta}_{\max}$ gives
\begin{equation}
|\tilde C_{11}| \leq \left(4IA^2 + \frac{m}{2}\right) \dot{\delta}_{\max}.
\end{equation}
Thus, \eqref{eq:Cbar_bound} holds with the obvious choice of $c_c$ as
\begin{equation}
c_c = \left(4IA^2 + \frac{m}{2}\right) \dot{\delta}_{\max}. \label{eq:bc_cbar_choice}
\end{equation}
\end{IEEEproof}

\begin{remark}\label{rem_scaled_norm}
Since the reduced velocity vector $\mathbf{v} = \begin{bmatrix}v_w & \dot{\delta}_f & \dot{\delta}_r\end{bmatrix}^{\text{T}}$ contains both translational and steering-rate components, the norm $\| \mathbf{v} \|$ used in the reduced-coordinate estimates is interpreted as a scaled reduced-coordinate norm. In particular, the steering-rate components are understood through the characteristic length scale of the robot, so that $v_w$, $a\dot{\delta}_f$, and $a\dot{\delta}_r$ define a dimensionally consistent velocity measure. Consequently, \eqref{eq_Cor_bound} represents a reduced-coordinate induced-norm bound, rather than a direct dimensional comparison of unscaled physical quantities.
\end{remark}

\begin{property}\label{property_5}
The matrix $\mathbf{\tilde{B}}(\mathbf{q})$ is invertible for all configurations.
\end{property}

\begin{IEEEproof}
For matrix $\mathbf{\tilde{B}}(\mathbf{q})$ defined by \eqref{eq_tildeB} the last two diagonal entries are $2>0$. For the first diagonal entry $\tilde{B}_{11}$, each term inside the brackets is nonnegative: $(\sin(\delta_f) - \sin(\delta_r))^2 \geq 0$, $1 + \cos(\delta_f - \delta_r) \geq 0$, and the constant $4>0$. Therefore $\tilde{B}_{11} \geq \frac{1}{r} \cdot 4 > 0$.
Thus all diagonal entries of $\mathbf{\tilde{B}}(\mathbf{q})$ are strictly positive, which implies that $\mathbf{\tilde{B}}(\mathbf{q})$ is nonsingular with inverse $\mathbf{\tilde{B}}^{-1}(\mathbf{q}) = \mathrm{diag}(\tilde{B}_{11}^{-1},\, 1/2,\, 1/2)$.
\end{IEEEproof}

\begin{property}\label{property_6}
If there exist nonnegative constants $c_k$ and $d_k$ defined in Assumption \ref{assum_f} such that \eqref{eq_assum_fd} is satisfied then
\begin{equation}\label{eq1_property6}
\| \mathbf{\tilde{f}} \| \leq \tilde{c} + \tilde{d} \| \mathbf{v} \|,
\end{equation}
with $\tilde{c} = \sigma_{J} \|\mathbf{c}\|$ and $\tilde{d} = \sigma_{J}^{2} \|\mathbf{d}\|$ where $\mathbf{c}$ and $\mathbf{d}$ are as in \eqref{eq_assum_fd}. Here $\sigma_{J}$ denotes the uniform induced gain of \eqref{eq1c_kinem_compact} under the weighted norm.
\end{property}

\begin{IEEEproof}
A weighted norm on the configuration derivative space is introduced as 
\begin{equation}
\|\mathbf{\dot{q}}\|_{W} = \sqrt{\mathbf{\dot{q}}^{\text{T}} \mathbf{W} \mathbf{\dot{q}}}
\end{equation} 
with $\mathbf{W} = \mathrm{diag}(1,1,1,r^2,1,1)$ which measures $r\dot{\varphi}$ instead of $\dot{\varphi}$ and thus avoids the amplification induced by $\dot{\varphi}=v_w/r$. The compatible induced matrix norm of \eqref{eq1c_kinem_compact} is defined by 
\begin{equation}
\|\mathbf{J}\|_{W} = \sup_{\mathbf{v}\neq 0} \frac{\|\mathbf{J} \mathbf{v}\|_{W}}{\|\mathbf{v}\|} = \sqrt{\lambda_{\max} \{\mathbf{J}^{\text{T}} \mathbf{W} \mathbf{J}\}},
\end{equation}
and the corresponding uniform bound is set as 
\begin{equation}
\sigma_{J} = \sup_{\mathbf{q} \in \mathcal{X}}\|\mathbf{J}\|_{W}.
\end{equation}
By definition, $\mathbf{\tilde{f}}(\mathbf{q}, \mathbf{v}) = \mathbf{J}^{\text{T}}(\mathbf{q}) \mathbf{f}(\mathbf{q}, \mathbf{J}(\mathbf{q}) \mathbf{v})$. Using submultiplicativity and \eqref{eq_assum_fd}, 
\begin{equation}
\|\mathbf{\tilde{f}}\| \leq \|\mathbf{J}^{\text{T}}\|_{W} \|\mathbf{f}\| \leq \|\mathbf{J}\|_{W} (\|\mathbf{c}\| + \|\mathbf{d}\| \|\dot{\mathbf{q}}\|_{W}),
\end{equation}
and since 
\begin{equation}
\|\mathbf{\dot{q}}\|_{W} = \|\mathbf{J}\mathbf{v}\|_{W} \leq \|\mathbf{J}\|_{W} \|\mathbf{v}\|,
\end{equation} 
it follows that 
\begin{equation}
\|\mathbf{\tilde{f}}\| \leq \sigma_{J}\|\mathbf{c}\| + \sigma_{J}^2\|\mathbf{d}\| \|\mathbf{v}\|,
\end{equation}
which yields \eqref{eq1_property6} with $\tilde{c}=\sigma_{J}\|\mathbf{c}\|$ and $\tilde{d}=\sigma_{J}^2\|\mathbf{d}\|$. Moreover, the induced gain in satisfies 
\begin{equation}
\|\mathbf{J}\|_{W} = \sqrt{\lambda_{\max} \{\mathbf{J}^{\text{T}} \mathbf{W} \mathbf{J}\}} = \max\{ \sqrt{S}, 1\},
\end{equation}
where $S = (1+\cos(\delta_f-\delta_r))/2 + A^2 (\sin(\delta_f) - \sin(\delta_r))^2 + 1$. Therefore, using $(1+\cos(\cdot))/2 \leq 1$ and $(\sin(\delta_f) - \sin(\delta_r))^2 \leq 4$, the uniform bound becomes $\sigma_{J} \leq \sqrt{2 + 4A^2}$. 

Note that the bound $\sigma_J$ is interpreted in the scaled coordinate metric induced by the preceding norm definition. Accordingly, the factor $A$ enters the induced-gain bound through this scaling so that $\sigma_J$ remains dimensionally consistent.
\end{IEEEproof}

Although the aforementioned properties are well known for constrained mechanical systems, they are presented here in order to derive expressions for estimating the coefficients $a_1$, $a_2$, $b_c$, $c_c$, $\tilde{c}$, $\tilde{d}$ and $\sigma_J$. The derivation shows that these coefficients can be estimated as functions of the mobile robot parameters and are independent of the state variables. These estimates provide explicit bounds for the reduced dynamics and uncertainty terms, and they are therefore used in the subsequent Lyapunov analysis to derive constructive gain conditions for the proposed controller. In particular, they enable sufficient conditions to be stated under which the closed-loop error dynamics are practically stable and the combined velocity-error and integral-error state is uniformly ultimately bounded.

\section{Controller design based on Lyapunov stability analysis}\label{sec_ControllerDesign}

\subsection{Control problem formulation}

This section formulates the synthesis problem for a proportional-integral controller augmented with feedforward compensation in closed loop with the mobile robot model given in \eqref{eq_dynamics_in_v}-\eqref{eq_tildefd}. The kinematic outer loop generates feasible reference velocities, whereas the main focus is on the inner-loop velocity controller. The robustness objective is to ensure reliable tracking of these reference velocities in the presence of configuration-dependent residual dynamics and bounded unmodelled nonlinearities. Accordingly, the control design aims to obtain a Lyapunov-based practical stability estimate and to guarantee uniform ultimate boundedness of the combined velocity-error and integral-error state.

\subsubsection{Virtual kinematic controller}\label{sec_Kin_controller}

Let the reference pose be $x_d(t)$, $y_d(t)$, $\theta_d(t)$ with desired yaw-rate $\omega_d(t)=\dot{\theta}_d(t)$. Define the tracking errors 
\begin{equation}\label{eq1_virt_kin}
\begin{aligned}
e_x &= (x_d - x) \cos(\theta) + (y_d - y) \sin(\theta),\\ 
e_y &= -(x_d - x) \sin(\theta) + (y_d - y) \cos(\theta),\\ 
e_{\theta} &= \theta_d - \theta.
\end{aligned}
\end{equation} 
Following the kinematic error-feedback principle in \cite{KKMN90}, but using the body-frame translational velocity compatible with the reduced kinematics \eqref{eq_kinem_compact}-\eqref{eq1c_kinem_compact}, the desired body-frame velocity components and the virtual yaw-rate command are defined as
\begin{equation}\label{eq_v_wd}
\begin{aligned}
u_x &= \dot{x}_d\cos(\theta)+\dot{y}_d\sin(\theta)+k_x e_x,\\ 
u_y &= -\dot{x}_d\sin(\theta)+\dot{y}_d\cos(\theta)+k_y e_y,\\
\omega_{virt} &= \omega_d + k_\theta e_\theta,
\end{aligned}
\end{equation} 
with $k_x, k_y,k_\theta > 0$. In contrast to a unicycle-type parametrization, the commands $u_x$ and $u_y$ represent the desired translational velocity components in the robot body frame. If the steering commands are feasible and the inner loop tracks the reduced velocity reference generated below, the translational velocity produced by the reduced kinematics matches $[u_x \; u_y]^{\text{T}}$, while the yaw rate matches $\omega_{virt}$ up to saturation and inner-loop tracking errors. Hence, in the nominal unsaturated case, the outer-loop tracking dynamics inherit the local exponential stability properties of the kinematic error-feedback law and are input-to-state stable with respect to inner-loop tracking errors and saturation-induced kinematic mismatch \cite{KKMN90}.

Using \eqref{eq_kinem_compact}-\eqref{eq1c_kinem_compact}, $\dot{\theta} = v_w A (\sin(\delta_f) - \sin(\delta_r))$. The steering references are obtained directly from the reduced kinematics. Let $v_t = \sqrt{u_x^2+u_y^2}$, $\beta_d = \operatorname{atan2}(u_y,u_x)$. The front and rear steering angles are parametrized as
\begin{equation}
\delta_{f,d} = \beta_d + \gamma_d,\quad \delta_{r,d} = \beta_d - \gamma_d,
\end{equation}
where $\gamma_d$ is selected so that the yaw-rate command is compatible with the reduced kinematics. Indeed, from \eqref{eq_kinem_compact}-\eqref{eq1c_kinem_compact} and the preceding parametrization one obtains
\begin{equation}
\begin{aligned}
u_x &= v_{w,d}\cos(\beta_d)\cos(\gamma_d),\\
u_y &= v_{w,d}\sin(\beta_d)\cos(\gamma_d),\\
\omega_{virt} &= 2A v_{w,d}\cos(\beta_d)\sin(\gamma_d).
\end{aligned}
\end{equation}
Therefore, the steering-shape variable is computed as
\begin{equation}
\gamma_d = \operatorname{atan2} \left( \omega_{virt}, 2A v_t \cos(\beta_d) \right),
\end{equation}
The desired common wheel tangential velocity is then defined by
\begin{equation}
v_{w,d} = \frac{v_t}{\cos(\gamma_d)}.
\end{equation}
The steering angles are saturated and unwrapped so that $|\delta_{(\cdot),d}| \leq \delta_{\max}$. If saturation of $\gamma_d$, $\delta_{f,d}$, or $\delta_{r,d}$ occurs, the achievable yaw-rate command is recomputed from $\omega_{virt, sat} = A v_{w,d} \left( \sin(\delta_{f,d}) - \sin(\delta_{r,d}) \right)$, and the corresponding saturation-induced kinematic mismatch is treated as a bounded perturbation in the subsequent stability analysis.

\begin{remark}
The mapping used to compute $\beta_d$, $\gamma_d$, and $v_{w,d}$ is understood in a regularized implementation sense. In particular, when $v_t$ approaches zero, the direction angle $\beta_d$ is kept at its previous filtered value or computed using a small numerical threshold, while $\gamma_d$ is saturated such that $|\gamma_d|\leq \gamma_{\max}<\pi/2$. Hence, $|\cos(\gamma_d)|\geq \cos(\gamma_{\max})>0$, and the definition of $v_{w,d}$ remains well posed. Any discrepancy introduced by this regularization and saturation is included in the bounded kinematic mismatch considered in the stability analysis.
\end{remark}

Track $\delta_{f,d}$ and $\delta_{r,d}$ with first-order laws 
\begin{equation}
\omega_{f,d} = \dot{\delta}_{f,d}-k_\delta(\delta_f - \delta_{f,d}),\quad \omega_{r,d} = \dot{\delta}_{r,d}-k_\delta(\delta_r-\delta_{r,d}), \label{eq_omega_fd}
\end{equation} 
with $k_\delta>0$ and $|\omega_{\cdot,d}| \leq \dot{\delta}_{\max}$. The derivatives $\dot{\delta}_{f,d}$ and $\dot{\delta}_{r,d}$ are obtained by differentiating the smoothed steering references or by filtered numerical differentiation. Finally, the desired reduced velocity vector supplied to the inner-loop dynamic controller is
\begin{equation}\label{eq_des_vel}
\mathbf{v}_d = \begin{bmatrix} v_{w,d} & \omega_{f,d} & \omega_{r,d} \end{bmatrix}^{\text{T}} .
\end{equation}

\begin{figure*}[!t]
\centering
\scalebox{0.75}{\input{Figures/control_scheme_kin_dyn_ver3.tex}}
\caption{Block diagram of the overall control system for the considered mobile robot.} \label{fig2_control_schematic}
\end{figure*}

\subsubsection{Dynamic controller}

A PI-type control law with feedforward compensation is considered in the following form
\begin{align}
\begin{split}
\bm{\tau} &= \mathbf{\tilde{B}}^{-1}(\mathbf{q}) \left( -{\mathbf{K}}_P \mathbf{e}_v - \mathbf{K}_I \bm{\eta}\right. \\
&+ \left. \mathbf{\tilde{M}}(\mathbf{q}_d)\mathbf{\dot{v}}_d + \mathbf{\tilde{C}}(\mathbf{q}_d, \mathbf{\dot{q}}_d)\mathbf{v}_d \right), 
\end{split}\label{PI_control_1}\\
\bm{\dot{\eta}} &= \mathbf{e}_v,\label{PI_control_2}
\end{align}
where $\mathbf{v}_d$ is the vector of desired reference velocities defined with \eqref{eq_des_vel}, $\mathbf{e}_v = \mathbf{v} - \mathbf{v}_d$ denotes the velocity tracking error vector, $\bm{\eta} \in \mathbb{R}^{3}$ is the integrator state, and $\mathbf{K}_P, \; \mathbf{K}_I \in \mathbb{R}^{3 \times 3}$ are positive-definite diagonal gain matrices. The matrices $\mathbf{\tilde{M}}$, $\mathbf{\tilde{C}}$ and $\mathbf{\tilde{B}}$, are defined by \eqref{eq_tildeM}, \eqref{eq_tildeC}, and \eqref{eq_tildeB}, respectively.

Considered robust control problem of a mobile robot is defined as follows.
\begin{problem}
Given the dynamic model \eqref{eq_dynamics_in_v}-\eqref{eq_tildefd} satisfying Properties \ref{property_1}-\ref{property_5} and subject to bounded uncertainty $\mathbf{\tilde{f}}(\mathbf{q}, \mathbf{v})$ satisfying Property \ref{property_6}, and given the desired reference velocity signal $\mathbf{v}_d(t)$ generated by the virtual kinematic controller \eqref{eq1_virt_kin}-\eqref{eq_omega_fd}, determine diagonal positive-definite gains ${\mathbf{K}}_P$ and ${\mathbf{K}}_I$ in \eqref{PI_control_1} such that the inner-loop velocity-error dynamics admit an augmented Lyapunov function yielding a practical stability estimate with strict decay terms. In particular, the objective is to guarantee uniform ultimate boundedness of the combined inner-loop error state $\mathbf{z} = \begin{bmatrix} \mathbf{e}_v^{\text{T}} & \bm{\eta}^{\text{T}}\end{bmatrix}^{\text{T}}$ under configuration-dependent residual dynamics and bounded unmodelled effects, for bounded feasible reference signals and bounded kinematic mismatch. In the nominal case, the same design should recover asymptotic stability of the inner-loop error dynamics while preserving the simple PI-like controller structure required for real-time implementation.
\end{problem}

\begin{remark}
Property \ref{property_6} models $\mathbf{\tilde{f}}(\mathbf{q},\mathbf{v})$ as a bounded memoryless nonlinearity on the operating set. Consequently, under nonzero residual dynamics and uncertainty terms, the control objective is formulated in terms of practical stability and uniform ultimate boundedness of the combined error state. Asymptotic convergence to the origin is recovered in the nominal case, when the residual and uncertainty induced terms vanish.
\end{remark}

The overall control system for the considered mobile robot is shown in Fig. \ref{fig2_control_schematic}. The outer-loop transforms the reference pose into feasible reduced velocity commands and desired steering motions, which are then tracked by the inner-loop. The inner loop applies the proposed PI-like controller with model-based feedforward compensation, using feedback from the robot states to generate actuator torques while explicitly accounting for the bounded uncertainty term $\mathbf{\tilde{f}}$.

\subsection{Stability analysis}

This section develops a Lyapunov-based stability analysis of the proposed PI-like controller. By using the residual error dynamics and the structural properties of the reduced robot model, sufficient gain conditions are derived for practical stability and uniform ultimate boundedness of the combined velocity-error and integral-error state in the presence of bounded uncertainties and configuration-dependent residual dynamics.

\subsubsection{Residual dynamics}

For a given desired velocity $\mathbf{v}_d(t)$ with associated desired state $\mathbf{q}_d(t)$ define
\begin{equation}\label{res_eq1}
\mathbf{u}_d = \mathbf{\tilde{M}}(\mathbf{q}_d) \mathbf{\dot{v}}_d + \mathbf{\tilde{C}}(\mathbf{q}_d, \mathbf{\dot{q}}_d) \mathbf{v}_d.
\end{equation}
Let the residual be
\begin{equation}\label{res_eq2}
\mathbf{r} = \mathbf{\tilde{M}}(\mathbf{q}) \mathbf{\dot{v}}_d + \mathbf{\tilde{C}}(\mathbf{q}, \mathbf{\dot{q}}) \mathbf{v}_d + \mathbf{\tilde{f}}(\mathbf{q}, \mathbf{v}_d) - \mathbf{u}_d.
\end{equation}
Then \eqref{eq_dynamics_in_v} can be rewritten in the following form
\begin{equation}\label{res_eq3}
\begin{aligned}
\mathbf{\tilde{M}}(\mathbf{q}) \mathbf{\dot{e}}_v + \mathbf{\tilde{C}}(\mathbf{q}, \mathbf{\dot{q}}) \mathbf{e}_v + \mathbf{r} &= \mathbf{\tilde{B}}(\mathbf{q}) {\bm{\tau}} - \mathbf{u}_d \\
&+ \mathbf{\tilde{f}}(\mathbf{q}, \mathbf{v}_d) - \mathbf{\tilde{f}}(\mathbf{q}, \mathbf{v}).
\end{aligned}
\end{equation}

Recall the structure of $\mathbf{\tilde{M}}(\mathbf{q})$ and $\mathbf{\tilde{C}}(\mathbf{q}, \mathbf{\dot{q}})$ from \eqref{eq_tildeM} and \eqref{eq_tildeC}, respectively, with dependencies through $\delta_f, \delta_r$ and $\dot{\delta}_f, \dot{\delta}_r$ on the operating domain $\delta_f, \delta_r \in \mathcal{D}_{\delta}$ and $|\dot{\delta}_f|,\; |\dot{\delta}_r| \leq \dot{\delta}_{\max}$. Let the reference signals be bounded by $V_d = \sup_t\|\mathbf{v}_d(t)\|$, $A_d = \sup_t \|\mathbf{\dot{v}}_d(t)\|$ and there are also the following properties.

\begin{property}\label{prop_LM_LC}
There exist positive Lipschitz constants $L_M$, $L_{C_1}$, $L_{C_2}$ such that
\begin{equation}
\|\mathbf{\tilde{M}}(\mathbf{q}_1) - \mathbf{\tilde{M}}(\mathbf{q}_2)\| \leq L_M \|\mathbf{q}_1 - \mathbf{q}_2\|, \label{eq1_LM_LC}
\end{equation}
\begin{equation}
\begin{aligned}
\|\mathbf{\tilde{C}}(\mathbf{q}_1, \mathbf{\dot{q}}_1) - \mathbf{\tilde{C}}(\mathbf{q}_2, \mathbf{\dot{q}}_2)\| &\leq L_{C_1} \|\mathbf{q}_1 - \mathbf{q}_2\| \\
&+ L_{C_2} \|\mathbf{\dot{q}}_1 - \mathbf{\dot{q}}_2\|, \label{eq2_LM_LC}
\end{aligned}
\end{equation}
for all $\mathbf{q}_1$, $\mathbf{\dot{q}}_1$, $\mathbf{q}_2$, $\mathbf{\dot{q}}_2$ in the operating domain.
\end{property}

\begin{IEEEproof}
Since $\mathbf{\tilde{M}}(\mathbf{q}_1)$ is diagonal matrix (see \eqref{eq_tildeM}) its induced $2$-norm equals the maximum absolute diagonal entry. Because only $\tilde{M}_{11}$ depends on $\delta_f$ and $\delta_r$, the Lipschitz constant $L_M$ is given by the supremum of the Euclidean norm of the gradient of the scalar map $\tilde{M}_{11}$. By the mean-value theorem and the bounds $|\sin(\cdot)| \leq 1$, $|\cos(\cdot)| \leq 1$, and $|\sin(\delta_f) - \sin(\delta_r)| \leq 2$, one obtains
\begin{equation}\label{eq3_LM_LC}
\begin{aligned}
L_M &= \sup_{\delta_f, \delta_r \in \mathcal D_\delta} \left\| \left[ {\textstyle \frac{\partial \tilde{M}_{11}}{\partial \delta_f} \;\; \frac{\partial \tilde{M}_{11}}{\partial \delta_r} } \right]^{\text{T}} \right\|\\ &= \sup_{\delta_f, \delta_r \in \mathcal D_\delta} {\textstyle \sqrt{\left(\frac{\partial \tilde{M}_{11}}{\partial \delta_f}\right)^2 + \left(\frac{\partial \tilde{M}_{11}}{\partial \delta_r}\right)^2}},\\
L_M &\leq \sqrt{2}\left(\frac{m}{2} + 4 I A^2\right).
\end{aligned}
\end{equation}

Similarly, for $L_{C_1}$ and $L_{C_2}$, the scalar $\tilde{C}_{11}$ (see \eqref{eq_tildeC}) is differentiated with respect to angles $\delta_f$, $\delta_r$ and steering rates $\dot{\delta}_f$ and $\dot{\delta}_r$, respectively. The Euclidean norms of the resulting derivative vectors are bounded using $|\sin(\cdot)| \leq 1$, $|\cos(\cdot)| \leq 1$ and $|\dot{\delta}_{(\cdot)}| \leq \dot{\delta}_{\max}$, and their suprema are adopted as $L_{C_1}$ and $L_{C_2}$, 
\begin{equation}\label{eq4_LM_LC}
\begin{aligned}
L_{C_1} &= \sup_{\delta_f, \delta_r\in\mathcal D_\delta, |\dot{\delta}_{(\cdot)}| \leq \dot{\delta}_{\max}} \left\| \left[ \frac{\partial \tilde{C}_{11}}{\partial \delta_f} \;\; \frac{\partial \tilde{C}_{11}}{\partial \delta_r} \right]^{\text{T}} \right\| \\
&= \sup_{\delta_f, \delta_r\in\mathcal D_\delta, |\dot{\delta}_{(\cdot)}| \leq \dot{\delta}_{\max}} { \textstyle \sqrt{\left(\frac{\partial \tilde{C}_{11}}{\partial \delta_f}\right)^2 + \left(\frac{\partial \tilde{C}_{11}}{\partial \delta_r}\right)^2}},\\
L_{C_1} &\leq \sqrt{2}\left(\frac{m}{4} + 4 I A^2 \right) \dot{\delta}_{\max},\\
L_{C_2} &= \sup_{\delta_f, \delta_r\in\mathcal D_\delta} \left\| \left[ \frac{\partial \tilde{C}_{11}}{\partial \dot{\delta}_f} \;\; \frac{\partial \tilde{C}_{11}}{\partial \dot{\delta}_r} \right]^{\text{T}} \right\|\\ &= \sup_{\delta_f, \delta_r\in\mathcal D_\delta} { \textstyle \sqrt{\left(\frac{\partial \tilde{C}_{11}}{\partial \dot{\delta}_f}\right)^2 + \left(\frac{\partial \tilde{C}_{11}}{\partial \dot{\delta}_r}\right)^2}},\\
L_{C_2} &\leq \sqrt{2} \left(\frac{m}{4} + 2 I A^2 \right),
\end{aligned}
\end{equation}
reflecting worst-case local slopes in configuration and velocity directions.
\end{IEEEproof}

\begin{property}\label{prop_F_d_v}
There exists a constant $d_v>0$ such that
\begin{equation}\label{eq1_prop_F_d_v}
\left\|\mathbf{\tilde{f}}(\mathbf{q}, \mathbf{v}_d) - \mathbf{\tilde{f}}(\mathbf{q}, \mathbf{v})\right\| \leq d_v \|\mathbf{v} - \mathbf{v}_d\|.	
\end{equation}
An explicit choice is 
\begin{equation}
d_v = \sup_{\mathbf{q}, \mathbf{\dot{q}} \in \mathcal{X} \times \mathcal{D}}\ \left\| \mathbf{J}(\mathbf{q})^{\text{T}} \frac{\partial \mathbf{f}}{\partial \mathbf{\dot{q}}} \mathbf{J}(\mathbf{q}) \right\|,
\end{equation} 
and the computable conservative bound 
\begin{equation}
d_v \leq \left(\sup_{\mathbf{q} \in \mathcal{X}}\|\mathbf{J}(\mathbf{q})\|\right)^{2} \left(\sup_{\mathbf{q}, \mathbf{\dot{q}} \in \mathcal{X} \times \mathcal{D}} \left\| \frac{\partial \mathbf{f}}{\partial \mathbf{\dot{q}}} \right\|\right) \leq \sigma^2_{J} L_{f_2},
\end{equation}
where $\mathbf{f}$ satisfies Assumption \ref{assum_f}, $\mathbf{J}$ is Jacobian \eqref{eq1c_kinem_compact}, $\sigma_{J}$ is bounded as in Property \ref{property_6}, and $L_{f_2}$ is defined as in Assumption \ref{assum_f}. Under Assumption \ref{assum_f}, $L_{f_2}$ is an upper bound on the induced spectral norm of $\partial \mathbf{f}/\partial \mathbf{\dot{q}}$ over $\mathcal{X} \times \mathcal{D}$.
\end{property}

\begin{IEEEproof}
Fix $\mathbf{q} \in \mathcal{X} \subset \mathbb{R}^{6}$ and define $\mathbf{G}(\mathbf{q}, \mathbf{v}) = \mathbf{\tilde{f}}(\mathbf{q}, \mathbf{v}) = \mathbf{J}^{\text{T}}(\mathbf{q}) \mathbf{f}(\mathbf{q}, \mathbf{J}(\mathbf{q}) \mathbf{v})$. By Assumption \ref{assum_f}, $\mathbf{G}$ is $C^1$ and the chain rule yields 
\begin{equation}
\frac{\partial \mathbf{G}}{\partial \mathbf{v}} = \mathbf{J}^{\text{T}}(\mathbf{q}) \frac{\partial \mathbf{f}}{\partial \mathbf{\dot{q}}} \mathbf{J}(\mathbf{q}).
\end{equation} 
By the mean-value theorem, there exists $\bm{\xi}$ on the line segment between $\mathbf{v}$ and $\mathbf{v}_d$ such that 
\begin{equation}
\begin{aligned}
\|\mathbf{\tilde{f}}(\mathbf{q}, \mathbf{v}_d) - \mathbf{\tilde{f}}(\mathbf{q}, \mathbf{v})\| &= \|\mathbf{G}(\mathbf{q}, \mathbf{v}_d) - \mathbf{G}(\mathbf{q}, \mathbf{v})\|\\ 
&\leq \left\| \left. \frac{\partial \mathbf{G}}{\partial \mathbf{v}}\right|_{\mathbf{v}=\bm{\xi}}\right\| \|\mathbf{v} - \mathbf{v}_d\|.
\end{aligned}
\end{equation}
Taking the supremum over $\mathbf{q}, \bm{\xi} \in \mathcal{X}$ yields the definition of $d_v$. Finally, using submultiplicativity of the induced operator norm, 
\begin{equation}
\left\| \mathbf{J}^{\text{T}}(\mathbf{q}) \frac{\partial \mathbf{f}}{\partial \mathbf{\dot{q}}} \mathbf{J}(\mathbf{q}) \right\| \leq \left\| \mathbf{J}(\mathbf{q}) \right\|^{2} \left\| \frac{\partial \mathbf{f}}{\partial \mathbf{\dot{q}}} \right\|,
\end{equation}
and taking suprema over $\mathcal X$ gives a bound on $d_v$.
\end{IEEEproof}

\begin{property}\label{prop_kin_bound}
For all $\mathbf{q}, \mathbf{q}_d, \mathbf{v}, \mathbf{v}_d \in \mathcal{X} \times \mathcal{X} \times \mathbb{R}^3 \times \mathbb{R}^3$, kinematic difference is bounded with
\begin{equation}\label{eq_kin_bound}
\|\mathbf{\dot{q}} - \mathbf{\dot{q}}_d\| \leq \sigma_{J} \|\mathbf{e}_v\| + \sigma_{\partial J} V_d \|\mathbf{q} - \mathbf{q}_d\|,
\end{equation}
where $\sigma_{J}$ is bounded as in Property \ref{property_6}, and $\sigma_{\partial J} \leq \sqrt{3/2 + 2A^2}$ where $A$ is defined in \eqref{eq2c_kinem_compact}.
\end{property}

\begin{IEEEproof}
From \eqref{eq_kinem_compact}, 
\begin{equation}
\begin{aligned}
\mathbf{\dot{q}} - \mathbf{\dot{q}}_d &= \mathbf{J}(\mathbf{q}) \mathbf{v} - \mathbf{J}(\mathbf{q}_d) \mathbf{v}_d \\
&= \mathbf{J}(\mathbf{q}) \mathbf{e}_v + ( \mathbf{J}(\mathbf{q}) - \mathbf{J}(\mathbf{q}_d) ) \mathbf{v}_d.
\end{aligned}
\end{equation}
Taking the Euclidean norm and using the triangle inequality, 
\begin{equation}
\begin{aligned}
\|\mathbf{\dot{q}} - \mathbf{\dot{q}}_d\| &\leq \|\mathbf{J}(\mathbf{q}) \mathbf{e}_v\| + \|(\mathbf{J}(\mathbf{q}) - \mathbf{J}(\mathbf{q}_d)) \mathbf{v}_d\|\\ 
&\leq \|\mathbf{J}(\mathbf{q})\| \|\mathbf{e}_v\| + \|\mathbf{J}(\mathbf{q}) - \mathbf{J}(\mathbf{q}_d)\| V_d.
\end{aligned}
\end{equation}
For the first term, use $\|\mathbf{J}(\mathbf{q})\| \leq \sigma_{J}$ to bound $\|\mathbf{J}(\mathbf{q})\mathbf{e}_v\|$. For the second term, by the mean-value estimate for matrix valued $C^1$ maps (Fr\'echet differentiability) there exists $\bm{\xi}$ on the segment joining $\mathbf{q}$ and $\mathbf{q}_d$ such that 
\begin{equation}
\begin{aligned}
\|\mathbf{J}(\mathbf{q}) - \mathbf{J}(\mathbf{q}_d)\| &\leq \sup_{\mathbf{q} \in [\mathbf{q}, \mathbf{q}_d]} \left\| \frac{\partial \mathbf{J}(\bm{\xi})}{\partial \mathbf{q}} \right\| \|\mathbf{q} - \mathbf{q}_d\|\\
&\leq \sigma_{\partial J} \|\mathbf{q} - \mathbf{q}_d\|.
\end{aligned}
\end{equation}
Substituting these bounds yields \eqref{eq_kin_bound}. 

For deriving the bound of $\sigma_{\partial J}$, the submultiplicativity, definition of the Fr\'echet derivative (the linear approximation) of a matrix valued function and its induced operator norm, triangle inequality and Cauchy-Schwarz-based quadratic bound are used, and based on elements of Jacobian matrix in \eqref{eq1c_kinem_compact}, yields 
\begin{equation}
\begin{aligned}
\sigma_{\partial J} &= \sup_{\mathbf{q} \in [\mathbf{q}, \mathbf{q}_d]} \left\| \frac{\partial \mathbf{J}}{\partial \mathbf{q}} \right\| = \sup_{\|\Delta \mathbf{q}\|=1} \left\| \sum_k \mathbf{J}_{,k} \Delta k \right\|\\ 
&\leq \left( \sum_k \| \mathbf{J}_{,k} \|^{2} \right)^{1/2} \leq \sqrt{3/2 + 2A^2},
\end{aligned}
\end{equation}
for $k \in \{ \theta,\; \delta_f,\; \delta_r \}$, where $\mathbf{J}_{,k}$ denotes the directional (Fr\'echet) derivative of $\mathbf{J}$ evaluated at $\mathbf{q}$ along the $k$-th coordinate direction of the configuration space.
\end{IEEEproof}

Using Properties \ref{property_6}, \ref{prop_LM_LC} and \ref{prop_kin_bound}, the norm of the residual defined by \eqref{res_eq2} can be estimated as follows
\begin{equation} \label{eq_R_full}
\begin{aligned}
\|\mathbf{r}\| &\leq \underbrace{\left(L_M\|\mathbf{\dot{v}}_d\| + L_{C_1}\|\mathbf{v}_d\| + L_{C_2} \sigma_{\partial J} \|\mathbf{v}_d\|^{2}\right)}_{A_q}\\ 
&\cdot \|\mathbf{q} - \mathbf{q}_d\| + \underbrace{L_{C_2}\sigma_J \|\mathbf{v}_d\|}_{A_v} \|\mathbf{e}_v\| + \underbrace{\left(\tilde{c} + \tilde{d} \|\mathbf{v}_d\|\right)}_{A_c}.
\end{aligned}
\end{equation}
Recalling the assumption that the desired references satisfy $\|\mathbf{v}_d(t)\|\ \leq V_d$, $\|\mathbf{\dot{v}}_d(t)\|\ \leq A_d$ for all $t \geq 0$, one can make the above coefficients explicit and time–invariant as follows
\begin{equation}\label{eq_R_coeffs}
\begin{aligned}
A_q &= L_M A_d + L_{C_1} V_d + L_{C_2} \sigma_{\partial J} V_d^2,\\ 
A_v &= L_{C_2} \sigma_J V_d,\\ 
A_c &= \tilde{c} + \tilde{d}\ V_d.
\end{aligned}
\end{equation}

\subsubsection{Construction of an augmented Lyapunov function for the inner-loop error state}

First, a baseline energy-like Lyapunov function is constructed following an approach similar to those reported in \cite{MKL23a, KZMJL16}. This baseline function is then augmented by a small cross term in order to obtain a Lyapunov function candidate for the state composed of the velocity error and the integral error.

Substituting \eqref{PI_control_1} into \eqref{res_eq3} yields the closed-loop equation in the following form
\begin{equation}
\begin{aligned}
\mathbf{\tilde{M}}(\mathbf{q}) \mathbf{\dot{e}}_v + \mathbf{\tilde{C}}(\mathbf{q}, \mathbf{\dot{q}}) \mathbf{e}_v + \mathbf{r} &= -{\mathbf{K}}_P \mathbf{e}_v - \mathbf{K}_I \bm{\eta}\\ 
&+ \mathbf{\tilde{f}}(\mathbf{q}, \mathbf{v}_d) - \mathbf{\tilde{f}}(\mathbf{q}, \mathbf{v}). \label{closed_loop_eq1}
\end{aligned}
\end{equation}

Then, treating the velocities as the output variables, \eqref{closed_loop_eq1} is multiplied by $\mathbf{e}^{\text{T}}_v$ to form the corresponding inner product
\begin{equation}
\begin{aligned}
&\mathbf{e}^{\text{T}}_v \mathbf{\tilde{M}}(\mathbf{q}) \mathbf{\dot{e}}_v + \mathbf{e}^{\text{T}}_v \mathbf{\tilde{C}}(\mathbf{q}, \mathbf{\dot{q}}) \mathbf{e}_v + \mathbf{e}^{\text{T}}_v \mathbf{r} =\\ 
&- \mathbf{e}^{\text{T}}_v {\mathbf{K}}_P \mathbf{e}_v - \mathbf{e}^{\text{T}}_v \mathbf{K}_I \bm{\eta} + \mathbf{e}^{\text{T}}_v \mathbf{\tilde{f}}(\mathbf{q}, \mathbf{v}_d) - \mathbf{e}^{\text{T}}_v \mathbf{\tilde{f}}(\mathbf{q}, \mathbf{v}). \label{lyap_eq1}\\
\end{aligned}
\end{equation}
Note that certain terms in \eqref{lyap_eq1} may be written as follows
\begin{align}
\mathbf{e}^{\text{T}}_v \mathbf{\tilde{M}}(\mathbf{q}) \mathbf{\dot{e}}_v &= \frac{d}{dt} \left( \frac{1}{2} \mathbf{e}^{\text{T}}_v \mathbf{\tilde{M}}(\mathbf{q}) \mathbf{e}_v \right) - \frac{1}{2} \mathbf{e}^{\text{T}}_v \mathbf{\dot{\tilde{M}}}(\mathbf{q}) \mathbf{e}_v,\label{lyap_eq2a}\\
\mathbf{e}^{\text{T}}_v \mathbf{K}_I \bm{\eta} &= \frac{d}{dt} \left( \frac{1}{2} \bm{\eta}^{\text{T}} \mathbf{K}_I \bm{\eta} \right),\label{lyap_eq2b}
\end{align}
with $\mathbf{K}_I = \mathbf{K}_I^{\text{T}}$ where \eqref{PI_control_2} has been used in \eqref{lyap_eq2b}. Based on the preceding terms and Property \ref{property_3}, the following is obtained from \eqref{lyap_eq1}
\begin{equation}
\begin{aligned}
&\frac{d}{dt} \left( \frac{1}{2} \mathbf{e}^{\text{T}}_v \mathbf{\tilde{M}}(\mathbf{q}) \mathbf{e}_v + \frac{1}{2} \bm{\eta}^{\text{T}} \mathbf{K}_I \bm{\eta} \right) =\\ 
&- \mathbf{e}^{\text{T}}_v {\mathbf{K}}_P \mathbf{e}_v - \mathbf{e}^{\text{T}}_v \mathbf{r} + \mathbf{e}^{\text{T}}_v \left(\mathbf{\tilde{f}}(\mathbf{q}, \mathbf{v}_d) - \mathbf{\tilde{f}}(\mathbf{q}, \mathbf{v})\right). \label{lyap_eq3}\\
\end{aligned}
\end{equation}
The expression in parentheses on the left-hand side of the nonlinear differential form \eqref{lyap_eq3} defines a baseline storage Lyapunov function candidate
\begin{equation}
V(\mathbf{e}_v, \bm{\eta}, \mathbf{q}) = \frac{1}{2} \mathbf{e}^{\text{T}}_v \mathbf{\tilde{M}}(\mathbf{q}) \mathbf{e}_v + \frac{1}{2} \bm{\eta}^{\text{T}} \mathbf{K}_I \bm{\eta}, \label{eq:V_base_new}
\end{equation}
whereas the right-hand side defines a candidate expression for its time derivative
\begin{equation}
\begin{aligned}
\frac{d}{dt} \left( V(\mathbf{e}_v, \bm{\eta}, \mathbf{q}) \right) &= - \mathbf{e}^{\text{T}}_v {\mathbf{K}}_P \mathbf{e}_v - \mathbf{e}^{\text{T}}_v \mathbf{r}\\ 
&+ \mathbf{e}^{\text{T}}_v \left(\mathbf{\tilde{f}}(\mathbf{q}, \mathbf{v}_d) - \mathbf{\tilde{f}}(\mathbf{q}, \mathbf{v})\right). \label{eq:Vdot_base}
\end{aligned}
\end{equation}

\begin{corollary}[Quadratic bounds on the baseline storage function] \label{corollary_energy_bounds}
Let $\mathbf{z} = \begin{bmatrix} \mathbf{e}_v^{\text{T}} & \bm{\eta}^{\text{T}} \end{bmatrix}^{\text{T}} \in \mathbb{R}^6$ and let $V$ be defined by \eqref{eq:V_base_new}. Then, for all admissible $q$,
\begin{equation}
\underline{\alpha}_V \|\mathbf{z}\|^2 \leq V(\mathbf{e}_v, \bm{\eta}, \mathbf{q}) \leq \overline{\alpha}_V \|\mathbf{z}\|^2, \label{eq:V_quadratic_bounds}
\end{equation}
where $\underline{\alpha}_V = \frac{1}{2} \min \left\{a_1, \lambda_{\min} \{\mathbf{K}_I\}\right\}$, $\overline{\alpha}_V = \frac{1}{2} \max \left\{a_2, \lambda_{\max} \{\mathbf{K}_I\}\right\}$. Hence, $V$ is positive definite and radially unbounded with respect to $z$ on the operating set.
\end{corollary}

\begin{IEEEproof}
The result follows directly from Property \ref{property_2} and from $\mathbf{K}_I = \mathbf{K}_I^{\text{T}} > 0$, which imply
\begin{equation}
a_1 \|\mathbf{e}_v\|^2 \leq \mathbf{e}_v^{\text{T}} \mathbf{\tilde{M}}(\mathbf{q}) \mathbf{e}_v \leq a_2 \|\mathbf{e}_v\|^2
\end{equation}
and
\begin{equation}
\lambda_{\min} \{\mathbf{K}_I\} \|\bm{\eta}\|^2 \leq \bm{\eta}^{\text{T}} \mathbf{K}_I \bm{\eta} \leq \lambda_{\max} \{\mathbf{K}_I\} \|\bm{\eta}\|^2.
\end{equation}
Combining these inequalities with \eqref{eq:V_base_new} yields \eqref{eq:V_quadratic_bounds}.
\end{IEEEproof}

The function $V$ provides a natural energy-like measure of the inner-loop error system. However, \eqref{eq:Vdot_base} does not contain an explicit negative-definite term in $\|\bm{\eta}\|^2$. To obtain a practical Lyapunov result for the combined inner-loop velocity-error and integral-error state $\mathbf{z} = \begin{bmatrix} \mathbf{e}_v^{\text{T}} & \bm{\eta}^{\text{T}} \end{bmatrix}^{\text{T}}$, the function $V$ is augmented by a small cross term.

\begin{proposition}[Augmented practical Lyapunov function with strict decay terms for the inner-loop error state]\label{prop_neg-semi}
Let
\begin{equation}
\mu_0 = \lambda_{\min}\{\mathbf{K}_P\} - A_v - d_v, \label{eq:mu0_def}
\end{equation}
and assume $\mu_0 > 0$. Define
\begin{equation}
W(\mathbf{e}_v, \bm{\eta},q) = V(\mathbf{e}_v, \bm{\eta},q) + \delta \mathbf{e}_v^{\text{T}} \mathbf{\tilde{M}}(\mathbf{q}) \bm{\eta}, \quad \delta > 0, \label{eq:W_def}
\end{equation}
and let
\begin{equation}
\chi = c_c + \lambda_{\max}\{\mathbf{K}_P\} + d_v + A_v. \label{eq:chi_def}
\end{equation}
If $\delta$ is chosen such that
\begin{equation}
\begin{split}
0 < \delta < \min &\left\{ \frac{\min \{a_1, \lambda_{\min} \{\mathbf{K}_I\}\}}{a_2},\right. \\ 
&\left. \frac{\mu_0}{2a_2 + 2\chi^2/\lambda_{\min}\{\mathbf{K}_I\}} \right\}, \label{eq:delta_condition}
\end{split}
\end{equation}
then there exist positive constants $\underline{\alpha}$,
$\overline{\alpha}$, $c_1$, $c_2$, $c_3$, and $c_4$ such that
\begin{equation}
\underline{\alpha} \|\mathbf{z}\|^2 \leq W(\mathbf{e}_v, \bm{\eta}, \mathbf{q}) \leq \overline{\alpha} \|\mathbf{z}\|^2, \label{eq:W_quadratic_bounds}
\end{equation}
and
\begin{equation}
\dot{W} \leq -c_1\|\mathbf{e}_v\|^2 - c_2\|\bm{\eta}\|^2 + c_3\|\mathbf{q} - \mathbf{q}_d\|^2 + c_4, \label{eq:Wdot_main}
\end{equation}
along all closed-loop trajectories. One explicit choice is
\begin{equation}
\begin{aligned}
\underline{\alpha} &= \frac{1}{2} \min \left\{ a_1 - \delta a_2,\; \lambda_{\min} \{\mathbf{K}_I\} - \delta a_2 \right\}, \\ 
\overline{\alpha} &= \frac{1}{2} \max \left\{ a_2 + \delta a_2,\; \lambda_{\max} \{\mathbf{K}_I\} + \delta a_2 \right\},\\
c_1 &= \frac{\mu_0}{2} - \delta \left( a_2 + \frac{\chi^2}{\lambda_{\min} \{\mathbf{K}_I\}} \right), \\ 
c_2 &= \frac{\delta}{4} \lambda_{\min} \{\mathbf{K}_I\}, \\ 
c_3 &= \frac{A_q^2}{\mu_0} + \delta \frac{A_q^2}{\lambda_{\min} \{\mathbf{K}_I\}}, \\ 
c_4 &= \frac{A_c^2}{\mu_0} + \delta \frac{A_c^2}{\lambda_{\min}\{\mathbf{K}_I\}}.
\end{aligned}
\end{equation}
\end{proposition}

\begin{IEEEproof}
First, by \eqref{eq:Vdot_base}, \eqref{eq1_prop_F_d_v}, and \eqref{eq_R_full},
\begin{equation}
\begin{aligned}
\dot{V} &= - \mathbf{e}_v^{\text{T}} \mathbf{K}_P \mathbf{e}_v - \mathbf{e}_v^{\text{T}} \mathbf{r} + \mathbf{e}_v^{\text{T}} \left( \mathbf{\tilde{f}}(\mathbf{q}, \mathbf{v}_d) - \mathbf{\tilde{f}}(\mathbf{q}, \mathbf{v}) \right)\\ 
&\leq -\lambda_{\min} \{\mathbf{K}_P\} \|\mathbf{e}_v\|^2 + \|\mathbf{e}_v\| \|\mathbf{r}\| + d_v\|\mathbf{e}_v\|^2 \\
&\leq -\mu_0\|\mathbf{e}_v\|^2 + A_q\|\mathbf{e}_v\| \|\mathbf{q} - \mathbf{q}_d\| + A_c\|\mathbf{e}_v\|. \label{eq:Vdot_prestrict}
\end{aligned}
\end{equation}

Next, define the cross term
\begin{equation}
S(\mathbf{e}_v, \bm{\eta},q) = \mathbf{e}_v^{\text{T}} \mathbf{\tilde{M}}(\mathbf{q}) \bm{\eta}.
\end{equation}
Differentiating $S$ along the trajectories of \eqref{closed_loop_eq1} and using $\bm{\dot{\eta}}=\mathbf{e}_v$ gives
\begin{equation}
\begin{aligned}
\dot{S} &= \mathbf{\dot{e}}_v^{\text{T}} \mathbf{\tilde{M}}(\mathbf{q}) \bm{\eta} + \mathbf{e}_v^{\text{T}} \mathbf{\dot{\tilde{M}}}(\mathbf{q}) \bm{\eta} + \mathbf{e}_v^{\text{T}} \mathbf{\tilde{M}}(\mathbf{q}) \bm{\dot{\eta}}\\
&= \left(-\mathbf{\tilde{C}}(\mathbf{q}, \mathbf{\dot{q}}) \mathbf{e}_v - \mathbf{r} -\mathbf{K}_P \mathbf{e}_v - \mathbf{K}_I \bm{\eta}\right. \\
&+ \left. \mathbf{\tilde{f}}(\mathbf{q}, \mathbf{v}_d) - \mathbf{\tilde{f}}(\mathbf{q}, \mathbf{v}) \right)^{\text{T}} \bm{\eta} \\
&+ \mathbf{e}_v^{\text{T}} \mathbf{\dot{\tilde{M}}}(\mathbf{q}) \bm{\eta} + \mathbf{e}_v^{\text{T}} \mathbf{\tilde{M}}(\mathbf{q}) \mathbf{e}_v. \label{eq:Sdot_step1}
\end{aligned}
\end{equation}
Using Remark \ref{remark_from_prop1_prop3}, i.e. $\mathbf{\dot{\tilde{M}}}(\mathbf{q}) = \mathbf{\tilde{C}}(\mathbf{q}, \mathbf{\dot{q}}) + \mathbf{\tilde{C}}(\mathbf{q}, \mathbf{\dot{q}})^{\text{T}}$, \eqref{eq:Sdot_step1} becomes
\begin{equation}
\begin{aligned}
\dot{S} &= \mathbf{e}_v^{\text{T}} \mathbf{\tilde{M}}(\mathbf{q}) \mathbf{e}_v - \bm{\eta}^{\text{T}} \mathbf{K}_I\bm{\eta} - \mathbf{e}_v^{\text{T}} \mathbf{K}_P \bm{\eta} - \mathbf{r}^{\text{T}} \bm{\eta}\\ 
&+ \left(\mathbf{\tilde{f}}(\mathbf{q}, \mathbf{v}_d) - \mathbf{\tilde{f}}(\mathbf{q}, \mathbf{v})\right)^{\text{T}} \bm{\eta} + \mathbf{e}_v^{\text{T}} \mathbf{\tilde{C}}(\mathbf{q}, \mathbf{\dot{q}}) \bm{\eta}. \label{eq:Sdot_step2}
\end{aligned}
\end{equation}
By Property \ref{property_2}, by the uniform bound on $\mathbf{\tilde{C}}(\mathbf{q}, \mathbf{\dot{q}})$ in Property \ref{property_4}, and by \eqref{eq1_prop_F_d_v} and \eqref{eq_R_full},
\begin{equation}
\begin{aligned}
\dot{S} &\leq a_2\|\mathbf{e}_v\|^2 - \lambda_{\min} \{\mathbf{K}_I\} \|\bm{\eta}\|^2 + \lambda_{\max} \{\mathbf{K}_P\} \|\mathbf{e}_v\| \|\bm{\eta}\| \\
&+ \left(A_q\|\mathbf{q} - \mathbf{q}_d\| + A_v\|\mathbf{e}_v\| + A_c\right) \|\bm{\eta}\| \\
&+ d_v\|\mathbf{e}_v\| \|\bm{\eta}\| + c_c\|\mathbf{e}_v\| \|\bm{\eta}\| \\
&\leq a_2\|\mathbf{e}_v\|^2 - \lambda_{\min} \{\mathbf{K}_I\} \|\bm{\eta}\|^2 + \chi \|\mathbf{e}_v\| \|\bm{\eta}\| \\
&+ A_q\|\mathbf{q} - \mathbf{q}_d\| \|\bm{\eta}\| + A_c\|\bm{\eta}\|, \label{eq:Sdot_bound}
\end{aligned}
\end{equation}
where $\chi$ is given by \eqref{eq:chi_def}.

Now, from \eqref{eq:W_def},
\begin{equation}
\dot{W} = \dot{V} + \delta \dot{S}.
\end{equation}
Substituting \eqref{eq:Vdot_prestrict} and \eqref{eq:Sdot_bound} yields
\begin{equation}
\begin{aligned}
\dot{W} &\leq -\mu_0 \|\mathbf{e}_v\|^2 + A_q\|\mathbf{e}_v\| \|\mathbf{q} - \mathbf{q}_d\| + A_c\|\mathbf{e}_v\| \\
&+ \delta a_2\|\mathbf{e}_v\|^2 - \delta \lambda_{\min} \{\mathbf{K}_I\} \|\bm{\eta}\|^2 + \delta \chi \|\mathbf{e}_v\| \|\bm{\eta}\| \\
&+ \delta A_q\|\mathbf{q} - \mathbf{q}_d\| \|\bm{\eta}\| + \delta A_c\|\bm{\eta}\|. \label{eq:Wdot_before_young}
\end{aligned}
\end{equation}
Apply Young's inequality in the following form:
\begin{equation}
\begin{aligned}
A_q\|\mathbf{e}_v\| \|\mathbf{q} - \mathbf{q}_d\| &\leq \frac{\mu_0}{4}\|\mathbf{e}_v\|^2 + \frac{A_q^2}{\mu_0} \|\mathbf{q} - \mathbf{q}_d\|^2,\\
A_c\|\mathbf{e}_v\| &\leq \frac{\mu_0}{4}\|\mathbf{e}_v\|^2 + \frac{A_c^2}{\mu_0},\\
\delta \chi \|\mathbf{e}_v\| \|\bm{\eta}\| &\leq \delta \frac{\chi^2}{\lambda_{\min} \{\mathbf{K}_I\}} \|\mathbf{e}_v\|^2 \\
&+ \delta \frac{\lambda_{\min} \{\mathbf{K}_I\}}{4} \|\bm{\eta}\|^2,\\
\delta A_q\|\mathbf{q} - \mathbf{q}_d\| \|\bm{\eta}\| &\leq \delta \frac{A_q^2}{\lambda_{\min} \{\mathbf{K}_I\}} \|\mathbf{q} - \mathbf{q}_d\|^2 \\
&+ \delta \frac{\lambda_{\min} \{\mathbf{K}_I\}}{4} \|\bm{\eta}\|^2,\\
\delta A_c\|\bm{\eta}\| &\leq \delta \frac{A_c^2}{\lambda_{\min} \{\mathbf{K}_I\}} \\
&+ \delta \frac{\lambda_{\min} \{\mathbf{K}_I\}}{4} \|\bm{\eta}\|^2.
\end{aligned}
\end{equation}
Substituting these inequalities into \eqref{eq:Wdot_before_young} gives \eqref{eq:Wdot_main} with the stated constants.

It remains to show \eqref{eq:W_quadratic_bounds}. By \eqref{eq:W_def} and Corollary \ref{corollary_energy_bounds},
\begin{equation}
W = V + \delta \mathbf{e}_v^{\text{T}}\mathbf{\tilde{M}}(\mathbf{q})\bm{\eta}.
\end{equation}
Using Property \ref{property_2},
\begin{equation}
\begin{aligned}
|\mathbf{e}_v^{\text{T}}\mathbf{\tilde{M}}(\mathbf{q})\bm{\eta}| &\leq \|\mathbf{\tilde{M}}(\mathbf{q})\| \|\mathbf{e}_v\| \|\bm{\eta}\|\\ 
&\leq a_2\|\mathbf{e}_v\| \|\bm{\eta}\| \leq \frac{a_2}{2} \left(\|\mathbf{e}_v\|^2 + \|\bm{\eta}\|^2\right).
\end{aligned}
\end{equation}
Combining the preceding bound with \eqref{eq:V_quadratic_bounds} yields \eqref{eq:W_quadratic_bounds} with the given $\underline{\alpha}$ and $\overline{\alpha}$. Under \eqref{eq:delta_condition}, both $\underline{\alpha}$ and $c_1$ are strictly positive, while $c_2>0$ is immediate. This completes the proof.
\end{IEEEproof}

\begin{remark}
The adjective ``strict'' refers to the negative-definite part of the decay estimate in \eqref{eq:Wdot_main}, where positive constants $c_1$ and $c_2$ multiply $\|\mathbf{e}_v\|^2$ and $\|\bm{\eta}\|^2$, respectively. In the nominal case, i.e., when the residual and uncertainty-induced additive terms vanish, this gives a strict Lyapunov decay condition for the inner-loop error state. In the general case, the additional bounded terms lead to a practical stability result and to uniform ultimate boundedness of $\mathbf{z} = \begin{bmatrix} \mathbf{e}_v^{\text{T}} & \bm{\eta}^{\text{T}} \end{bmatrix}^{\text{T}}$.
\end{remark}

\subsubsection{Practical stability and uniform ultimate boundedness of the inner-loop error state}

The previous proposition establishes that the augmented function $W$ provides a practical Lyapunov estimate for the combined velocity-error and integral-error state, with strict decay terms and explicitly bounded additive terms associated with the kinematic mismatch $\|\mathbf{q} - \mathbf{q}_d\|$ and the constant residual level $A_c$. This estimate yields the following practical stability result.

\begin{proposition}[Uniform ultimate boundedness of the inner-loop error state]\label{prop_uub} 
Suppose that the conditions of Proposition \ref{prop_neg-semi} hold and that the reference signals generated by the virtual kinematic controller are bounded on the considered operating set. Furthermore, assume that there exists a constant $\bar{q} \geq 0$ such that
\begin{equation}
\|\mathbf{q}(t) - \mathbf{q}_d(t)\| \leq \bar{q}, \quad \forall t \geq 0. \label{eq:qerr_bounded}
\end{equation}
Let
\begin{equation}
c = \min\{c_1, c_2\}, \quad \bar{\rho} = c_3 \bar{q}^2 + c_4.
\end{equation}
Then, along all closed-loop trajectories,
\begin{equation}
\dot{W} \leq -c\|\mathbf{z}\|^2 + \bar{\rho}, \label{eq:Wdot_uub}
\end{equation}
and, equivalently,
\begin{equation}
\dot{W} \leq -\frac{c}{\overline{\alpha}} W + \bar{\rho}. \label{eq:Wdot_linear}
\end{equation}
Hence,
\begin{equation}
W(t) \leq e^{-\frac{c}{\overline{\alpha}} t} W(0) + \frac{\overline{\alpha}}{c} \bar{\rho} \left(1 - e^{-\frac{c}{\overline{\alpha}} t}\right), \quad \forall t \geq 0, \label{eq:W_solution_bound}
\end{equation}
and the inner-loop error state $\mathbf{z} = \begin{bmatrix} \mathbf{e}_v^{\text{T}} & \bm{\eta}^{\text{T}} \end{bmatrix}^{\text{T}}$ is uniformly ultimately bounded. In particular,
\begin{equation}
\limsup_{t \to \infty}\|\mathbf{z}(t)\| \leq \sqrt{\frac{\overline{\alpha}}{\underline{\alpha}} \frac{\bar{\rho}}{c}}. \label{eq:z_ultimate_bound}
\end{equation}
Moreover, in the nominal case for which $\mathbf{q}(t) \equiv \mathbf{q}_d(t)$ and $A_c = 0$, one has $\bar{\rho} = 0$ and therefore the equilibrium $\mathbf{z} = 0$ is exponentially stable on the operating set.
\end{proposition}

\begin{IEEEproof}
By Proposition 1 and \eqref{eq:qerr_bounded},
\begin{equation}
\dot{W} \leq -c_1\|\mathbf{e}_v\|^2 - c_2\|\bm{\eta}\|^2 + c_3\bar{q}^2 + c_4 \leq -c\|\mathbf{z}\|^2 + \bar{\rho},
\end{equation}
which proves \eqref{eq:Wdot_uub}. Using the upper bound in \eqref{eq:W_quadratic_bounds},
\begin{equation}
W \leq \overline{\alpha} \|\mathbf{z}\|^2 \; \Longrightarrow \; \|\mathbf{z}\|^2 \geq \frac{1}{\overline{\alpha}} W.
\end{equation}
Substituting this into \eqref{eq:Wdot_uub} yields \eqref{eq:Wdot_linear}. The comparison lemma then gives \eqref{eq:W_solution_bound}. Finally, using the lower bound in \eqref{eq:W_quadratic_bounds},
\begin{equation}
\underline{\alpha} \|\mathbf{z}\|^2 \leq W,
\end{equation}
and taking $\limsup_{t\to\infty}$ in \eqref{eq:W_solution_bound}, one obtains \eqref{eq:z_ultimate_bound}.

If $\mathbf{q}(t) \equiv \mathbf{q}_d(t)$ and $A_c = 0$, then $\bar{\rho} = 0$, so \eqref{eq:Wdot_linear} reduces to
\begin{equation}
\dot W \leq -\frac{c}{\overline{\alpha}} W,
\end{equation}
which implies exponential decay of $W(t)$ to zero. By \eqref{eq:W_quadratic_bounds}, this is equivalent to exponential convergence of $\mathbf{z}(t)$ to the origin.
\end{IEEEproof}

\begin{remark}
The stability result in Proposition \ref{prop_uub} is an inner-loop result. It establishes practical stability and uniform ultimate boundedness of the combined velocity-error and integral-error state $\mathbf{z} = \begin{bmatrix} \mathbf{e}_v^{\text{T}} & \bm{\eta}^{\text{T}} \end{bmatrix}^{\text{T}}$ for bounded feasible reference signals and bounded kinematic mismatch $\|\mathbf{q}(t)-\mathbf{q}_d(t)\|$. Therefore, it should not be interpreted as a stand-alone proof of global trajectory-tracking stability of the complete outer-loop/inner-loop cascade. In the overall system, inner-loop tracking errors and saturation-induced kinematic mismatch act as bounded perturbations to the virtual kinematic controller, which is consistent with the practical stability objective considered in this work.
\end{remark}

\section{Experimental validation}\label{sec_Results}

\subsection{Experimental setup and controller implementation}
The mobile robot used for experimental validation was originally designed for vertical-surface locomotion and nondestructive testing. Its structure combines carbon-fiber tubes and 3D-printed acrylonitrile styrene acrylate (ASA) components, providing a high stiffness-to-mass ratio. The robot is actuated by eight Dynamixel XH430-W210 smart servo actuators, four assigned to wheel steering and four to wheel driving, enabling quasi-omnidirectional motion and stable orientation control during vertical operation. The actuators expose a torque (current) control mode that can be commanded at high update rates. The hybrid adhesion system comprises an electric ducted fan and a drone-type propulsion unit arranged near the robot's center of gravity; this configuration enhances adhesion and compensates for gravitational loading during vertical motion.

The low-level control system runs on an NVIDIA Jetson NX onboard computer executing ROS 2 Humble in a containerized (Docker) environment. The \texttt{ros2\_control} framework provides a hardware abstraction layer for coordinated velocity control of the drive and steering wheels. The onboard computer interfaces with the actuators through a U2D2 communication converter, which bridges the host and the motor bus. The actuators' half-duplex TTL serial protocol limits the control loop to an update rate of 500~Hz. This rate is sufficient to track the commanded wheel-velocity references while using motor torque as the control input.
Fig.~\ref{experimental_setup} illustrates the experimental configurations of the robot on the vertical and horizontal test surfaces.

The robot operates at 24~V, supplied through a power cable (labeled \textit{power cable} in Fig.~\ref{experimental_setup}), while the Ethernet cable is used only for logging data. During experiments on vertical surfaces, the robot is secured with a safety rope to prevent damage in the event of a failure.

\begin{figure}[!t]
\centering 
\includegraphics[width=\linewidth]{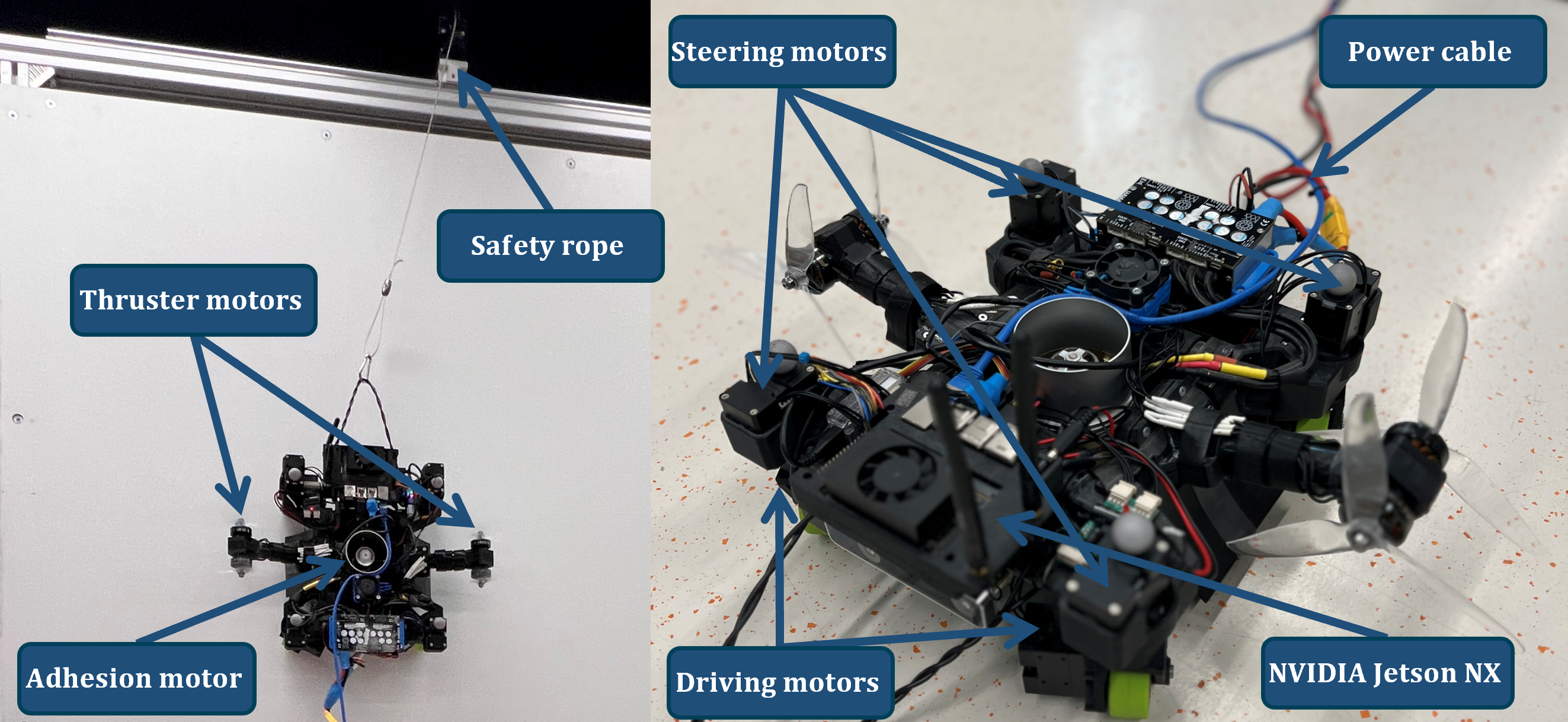} 
\caption{Experimental setup of the mobile robot on a vertical surface (left) and a horizontal surface (right).} \label{experimental_setup}
\end{figure}

\begin{table}[!t]
\caption{Robot physical parameters.} \label{tab:robot_parameters}
\centering
\setlength{\tabcolsep}{5pt}
\begin{tabular}{l c c}
\hline
\textbf{Parameter} & \textbf{Description} & \textbf{Value} \\
\hline
$r$ & Wheel radius & $0.0254~\mathrm{m}$ \\
$a$ & Half-length of robot body & $0.1125~\mathrm{m}$ \\
$b$ & Half-width of robot body & $0.1125~\mathrm{m}$ \\
$m$ & Robot mass & $3.50~\mathrm{kg}$ \\
$m_w$ & Wheel mass & $0.03203~\mathrm{kg}$ \\
$I_\theta$ & Robot yaw inertia & $0.03333~\mathrm{kgm^2}$ \\
$I_{\varphi}$ & Wheel rotational inertia & $1.03 \times 10^{-5}~\mathrm{kgm^2}$ \\
$I_\delta$ & Steering inertia & $0.002~\mathrm{kgm^2}$ \\
$A$ & $a/(2a^2+2b^2)$ & $2.2222~\mathrm{m^{-1}}$ \\
$I$ & $I_{\theta}+4m_w(a^2+b^2)$ & $0.0365~\mathrm{kgm^2}$ \\
\hline
\end{tabular}
\end{table}

Table \ref{tab:robot_parameters} summarizes the physical parameters of the robot. From these values, the inertia bounds from Property \ref{property_2} are $a_1 = 0.0040$, $a_2 = 4.2848$, in the reduced-coordinate inertia units. Furthermore, the gain-selection constants are $\sigma_J=4.6640$ (see Property \ref{property_6}), $L_{C_2}=1.7479\, \mathrm{kg}$ (see Property \ref{prop_LM_LC}), $|v_w| \leq v_{w, \max}=0.13\, \mathrm{m/s}$, and $|\dot{\delta}_{(\cdot)}| \leq \dot{\delta}_{\max}=\pi/2\, \mathrm{rad/s}$. Hence, 
\[
V_d \leq \sqrt{v_{w, \max}^2 + 2a^2 \dot{\delta}_{\max}^2} = 0.2817\, \mathrm{m/s}.
\]
Consistently with the scaled reduced-coordinate interpretation introduced in Remark \ref{rem_scaled_norm}, the reference-velocity bound is evaluated using a weighted Euclidean norm,
\[
\|\mathbf{v}_d\|_{W}^2 = v_{w,d}^2 + (a\dot{\delta}_{f,d})^2 + (a\dot{\delta}_{r,d})^2 \leq v_{w,\max}^2 + 2a^2\dot{\delta}_{\max}^2 .
\]
This norm represents the reduced velocity in a dimensionally consistent form, since the steering-rate components are scaled by the characteristic length $a$. Accordingly, the scalar constants used below in the gain conditions are interpreted in the same scaled reduced-coordinate representation, whereas the implemented diagonal gains are reported in their channel-wise physical units.

Furthermore, $A_v = 2.2964\, \mathrm{kg\,m/s}$ (see \eqref{eq_R_coeffs}). Note that $\dot{\delta}_{\max}$ is a steering-reference design limit, not the physical actuator saturation. Hence, the manufacturer-declared peak rate ($\approx 5.24\, \mathrm{rad/s}$) is not used in gain selection, since the analysis requires a bound consistent with the admissible reference signals.

The quadratic Coriolis-growth constant from Property \ref{property_4} is $b_c = 1.2355\,\mathrm{kg}$. Using the uniform Coriolis-matrix bound from Property \ref{property_4} gives $c_c = 3.8814\, \mathrm{kg/s}$.

Viscous friction in the actuation is modelled with coefficient $b_f=0.0305\, \mathrm{Nms/rad}$, which accounts for the motor-reducer losses and wheel-ground contact effects. Accordingly, $L_{f_2}=b_f=0.0305\, \mathrm{Nms/rad}$ (see Assumption \ref{assum_f}) and $d_v=0.6635\, \mathrm{Nms/rad}$ (see Property \ref{prop_F_d_v}). Therefore, according to \eqref{eq:mu0_def}, the sufficient strict decay condition in the Lyapunov analysis requires
\[
\mu_0 = \lambda_{\min}\{\mathbf{K}_P^\tau\} - A_v - d_v > 0,
\]
where $\mathbf{K}_P^\tau$ denotes the proportional gain matrix in the torque domain. Hence, the required lower bound is
\[
\lambda_{\min}\{\mathbf{K}_P^\tau\} > A_v + d_v = 2.9599.
\]

For implementation in current control mode, the torque domain gains are mapped using the motor torque constant $K_t = 1.923\, \mathrm{Nm/A}$. Thus, $\mathbf{K}_P^\tau = K_t \mathbf{K}_{P}^i$, $\mathbf{K}_I^\tau = K_t \mathbf{K}_{I}^i$, where $\mathbf{K}_{P}^i$ and $\mathbf{K}_{I}^i$ are the implemented current domain gains. These gain matrices were fine-tuned to improve the transient response while preserving the strict decay condition. The finally implemented current domain gain matrices are
\[
\mathbf{K}_{P}^i = \mathrm{diag} \left(1.563 \frac{\mathrm{As}}{\mathrm{m}^2}, 2.344 \frac{\mathrm{As}}{\mathrm{rad}}, 2.344 \frac{\mathrm{As}}{\mathrm{rad}}\right)
\]
and
\[
\mathbf{K}_{I}^i = \mathrm{diag} \left(0.061 \frac{\mathrm{A}}{\mathrm{m}^2}, 0.092 \frac{\mathrm{A}}{\mathrm{rad}}, 0.092 \frac{\mathrm{A}}{\mathrm{rad}}\right).
\]
The corresponding reduced space gain matrices scaled by the motor torque constant are
\[
\mathbf{K}_{P}^{\tau} = \mathrm{diag} \left(3.0056 \frac{\mathrm{Ns}}{\mathrm{m}}, 4.5075 \frac{\mathrm{Nms}}{\mathrm{rad}}, 4.5075 \frac{\mathrm{Nms}}{\mathrm{rad}}\right)
\]
and
\[
\mathbf{K}_{I}^{\tau} = \mathrm{diag} \left(0.1173 \frac{\mathrm{N}}{\mathrm{m}}, 0.1769 \frac{\mathrm{Nm}}{\mathrm{rad}}, 0.1769 \frac{\mathrm{Nm}}{\mathrm{rad}}\right).
\]

Consequently, in the adopted reduced-coordinate representation, the numerical diagonal gain coefficients give $\lambda_{\min}\{\mathbf{K}_{P}^i\}=1.563$, $\lambda_{\max}\{\mathbf{K}_{P}^i\}=2.344$, and $\lambda_{\min}\{\mathbf{K}_{P}^{\tau}\}=3.0056$, $\lambda_{\max}\{\mathbf{K}_{P}^{\tau}\}=4.5075$. The corresponding current domain gain condition is therefore
\[
\lambda_{\min}\{\mathbf{K}_{P}^i\} > \frac{A_v+d_v}{K_t} = 1.539,
\]
which is satisfied by the implemented gains. In particular,
\[
\mu_0 = K_t\lambda_{\min}\{\mathbf{K}_{P}^i\} - A_v - d_v = 0.0457 > 0.
\]

Furthermore, $\lambda_{\min}\{\mathbf{K}_{I}^{\tau}\}=0.1173$, $\lambda_{\max}\{\mathbf{K}_{I}^{\tau}\}=0.1769$, which confirms that $\mathbf{K}_{I}^{\tau}>0$, as required for the practical-stability analysis based on the augmented Lyapunov function. The constant entering the cross-term bound in Proposition \ref{prop_neg-semi} is
\[
\chi = c_c + \lambda_{\max}\{\mathbf{K}_P^\tau\} + d_v + A_v = 11.3488 \frac{\mathrm{kg}}{\mathrm{s}}.
\]

Using these values in \eqref{eq:delta_condition} gives the conservative admissible range $0 < \delta < \bar{\delta}$, $\bar{\delta} = 2.0752\times 10^{-5}\, \mathrm{s^{-1}}$. Note that, the auxiliary scalar $\delta$ is not an implementation parameter or a controller gain. It is used only in the Lyapunov function to prove practical stability.

The outer-loop gains were selected as
\[
k_x = k_y = k_\theta = k_\delta=5\, \mathrm{s^{-1}}.
\]

\subsection{Experimental results and discussion}

Three experimental scenarios are considered:
\begin{itemize}
\item[\textit{i)}] Lissajous-trajectory tracking on a horizontal surface with disturbance injection using the proposed Lyapunov-based PI-type control law with feedforward compensation, compared with a PI controller and a sliding-mode controller;
\item[\textit{ii)}] planar flower-trajectory tracking on a horizontal surface with disturbance injection using the proposed Lyapunov-based PI-type control law with feedforward compensation; and
\item[\textit{iii)}] Lissajous-trajectory tracking on a vertical surface under gravity and adhesion-driven contact effects using the proposed Lyapunov-based PI-type control law with feedforward compensation.
\end{itemize}

For a experimental comparison \textit{i)}, all controllers were evaluated using the same reference generator and the same measured feedback signals. The baseline PI and sliding-mode controllers were implemented in the same reduced velocity space as the proposed controller. Their gains were tuned experimentally to obtain the best tracking performance. The implemented control laws and gain values are summarized below.
The proposed controller achieves lower tracking error in the x-y plane, together with smoother velocity responses and reduced actuation effort, as summarized in Table~\ref{tab:results} where \textit{No. Dist} presents the results without distrubance, \textit{Norm. Dist} presents the results with a normal disturbance, and \textit{Norm.-Ax. Dist} presents the results with a normal and axial disturbance. The tracking performance is evaluated in terms of root mean square error (RMSE) of the position and orientation tracking errors, as well as the total variation (TV) of the velocity commands and the control inputs, which are indicative of the smoothness of the control actions.

\begin{table*}[!t]
\caption{Tracking performance across the three experiments.}
\label{tab:results}
\centering
\fontsize{8}{9}\selectfont

\begin{minipage}[t]{0.49\textwidth}
\vspace{0pt}
\centering
\begin{tabular}{|c|c|c|c|c|}
\hline
\textbf{Controller} & \textbf{RMSE}
& \textbf{No Dist.} & \textbf{Norm. Dist.} & \textbf{Norm.-Ax. Dist.}\\
\hline\hline
\multicolumn{5}{|l|}{\textit{Experiment i) -- Lissajous trajectory on horizontal surface}}\\
\hline
\multirow[c]{6}{*}{Lyap. PI\,+\,FF}
& $e_x$           & \textbf{0.008} & \textbf{0.009} & \textbf{0.010}\\
\cline{2-5}
& $e_y$           & \textbf{0.010} & \textbf{0.011} & \textbf{0.012}\\
\cline{2-5}
& $e_\theta$      & \textbf{0.082} & 0.085 & \textbf{0.086}\\
\cline{2-5}
& $\mathrm{TV}_v$ & \textbf{7.76}  & \textbf{8.45}  & \textbf{15.54}\\
\cline{2-5}
& $\mathrm{TV}_f$ & \textbf{16.64} & \textbf{15.09} & \textbf{16.14}\\
\cline{2-5}
& $\mathrm{TV}_r$ & \textbf{17.65} & \textbf{12.59} & \textbf{18.12}\\
\hline
\multirow[c]{6}{*}{PI}
& $e_x$           & 0.010  & 0.010  & 0.018\\
\cline{2-5}
& $e_y$           & 0.011  & 0.011  & 0.017\\
\cline{2-5}
& $e_\theta$      & 0.083  & \textbf{0.084}  & 0.097\\
\cline{2-5}
& $\mathrm{TV}_v$ & 73.853 & 65.364 & 47.321\\
\cline{2-5}
& $\mathrm{TV}_f$ & 26.539 & 17.823 & 22.148\\
\cline{2-5}
& $\mathrm{TV}_r$ & 27.688 & 13.214 & 31.011\\
\hline
\multirow[c]{6}{*}{SMC}
& $e_x$           & 0.020  & 0.023  & 0.028\\
\cline{2-5}
& $e_y$           & 0.020  & 0.022  & 0.028\\
\cline{2-5}
& $e_\theta$      & 0.104  & 0.104  & 0.140\\
\cline{2-5}
& $\mathrm{TV}_v$ & 27.224 & 22.249 & 30.961\\
\cline{2-5}
& $\mathrm{TV}_f$ & 47.280 & 56.291 & 57.266\\
\cline{2-5}
& $\mathrm{TV}_r$ & 59.335 & 56.014 & 64.428\\
\hline
\end{tabular}%
\end{minipage}
\hfill
\begin{minipage}[t]{0.49\textwidth}
\vspace{0pt}
\centering
\begin{tabular}{|c|c|c|c|c|}
\hline
\textbf{Controller} & \textbf{RMSE}
& \textbf{No Dist.} & \textbf{Norm. Dist.} & \textbf{Norm.-Ax. Dist.}\\
\hline\hline
\multicolumn{5}{|l|}{\textit{Experiment ii) -- Flower trajectory on horizontal surface}}\\
\hline
\multirow[c]{6}{*}{Lyap. PI\,+\,FF}
& $e_x$           & 0.077  & 0.080   & 0.077\\
\cline{2-5}
& $e_y$           & 0.069  & 0.070   & 0.068\\
\cline{2-5}
& $e_\theta$      & 0.054  & 0.083   & 0.059\\
\cline{2-5}
& $\mathrm{TV}_v$ & 16.322 & 16.936  & 39.227\\
\cline{2-5}
& $\mathrm{TV}_f$ & 81.343 & 166.548 & 309.920\\
\cline{2-5}
& $\mathrm{TV}_r$ & 42.979 & 34.573  & 52.122\\
\hline
\end{tabular}%

\vspace{2mm}

\begin{tabular}{|c|c|c|c|c|}
\hline
\textbf{Controller} & \textbf{RMSE}
& \textbf{No Dist.} & \textbf{Norm. Dist.} & \textbf{Norm.-Ax. Dist.}\\
\hline\hline
\multicolumn{5}{|l|}{\textit{Experiment iii) -- Lissajous trajectory on vertical surface}}\\
\hline
\multirow[c]{6}{*}{Lyap. PI\,+\,FF}
& $e_x$           & -- & -- & 0.044\\
\cline{2-5}
& $e_y$           & -- & -- & 0.130\\
\cline{2-5}
& $e_\theta$      & -- & -- & 0.000\\
\cline{2-5}
& $\mathrm{TV}_v$ & -- & -- & 16.095\\
\cline{2-5}
& $\mathrm{TV}_f$ & -- & -- & 47.999\\
\cline{2-5}
& $\mathrm{TV}_r$ & -- & -- & 24.550\\
\hline
\end{tabular}%
\end{minipage}

\end{table*}

\begin{figure*}[!t]
\centering
\includegraphics[scale=0.43]{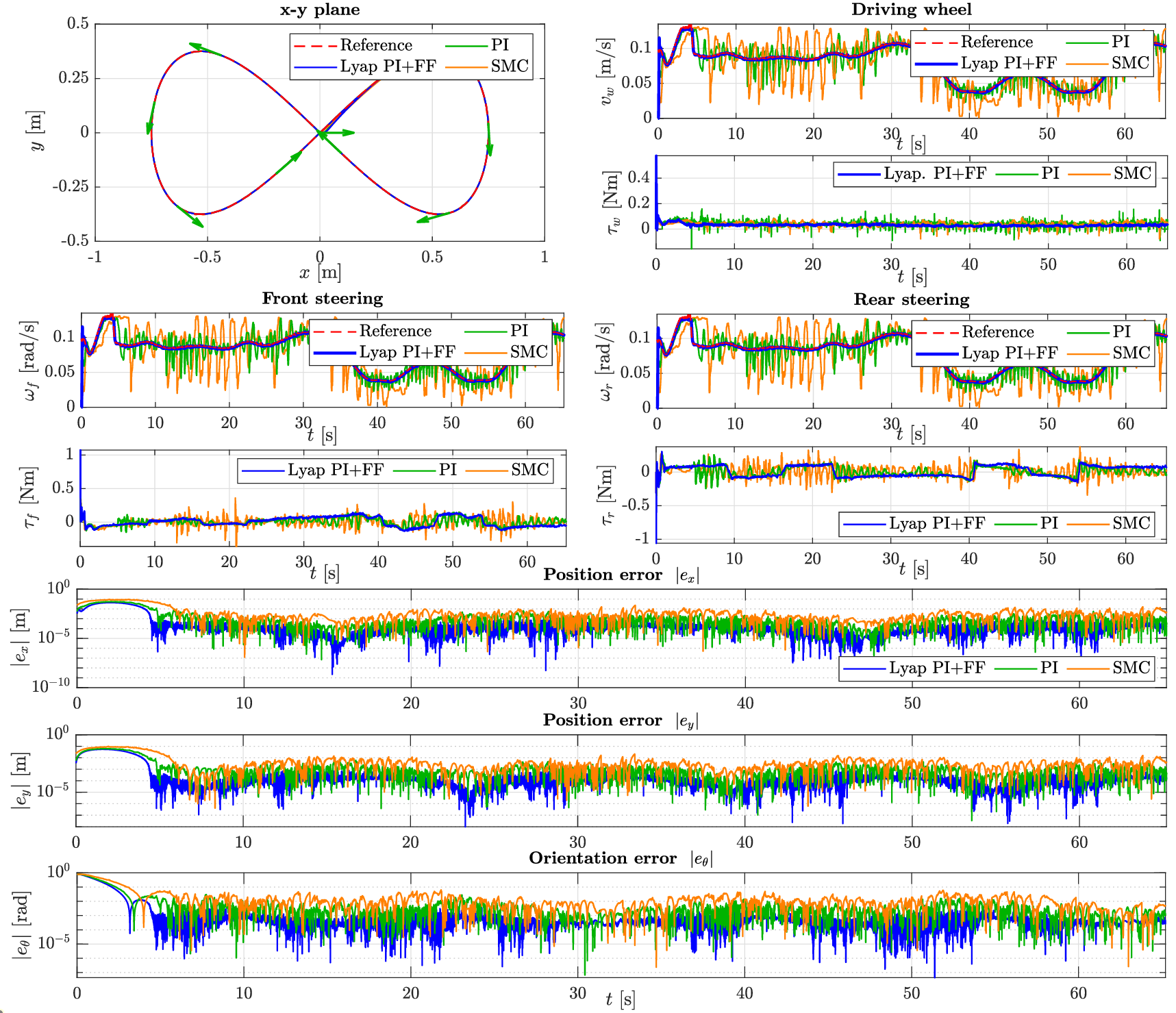}
\caption{Experimental responses of the robot without disturbance, compared across the three different controllers. Green arrows indicate robot orientation.}\label{fig_comp_no_dist}
\end{figure*}

\begin{figure*}[!t]
\centering
\includegraphics[scale=0.43]{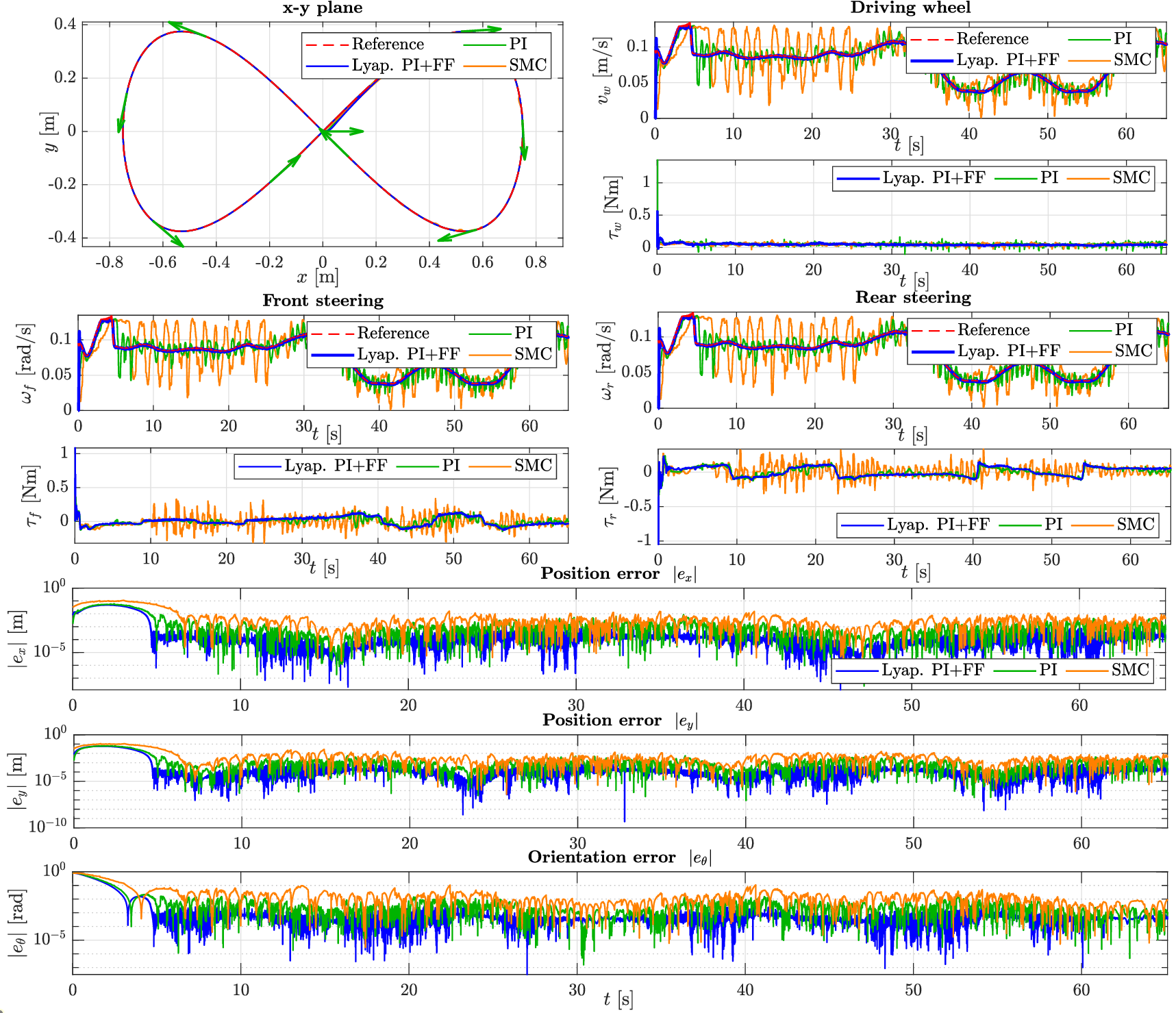}
\caption{Experimental responses of the robot subjected to a normal disturbance, compared across the three different controllers. Green arrows indicate robot orientation.}\label{fig_comp_normal_dist}
\end{figure*}

\begin{figure*}[!t]
\centering
\includegraphics[scale=0.43]{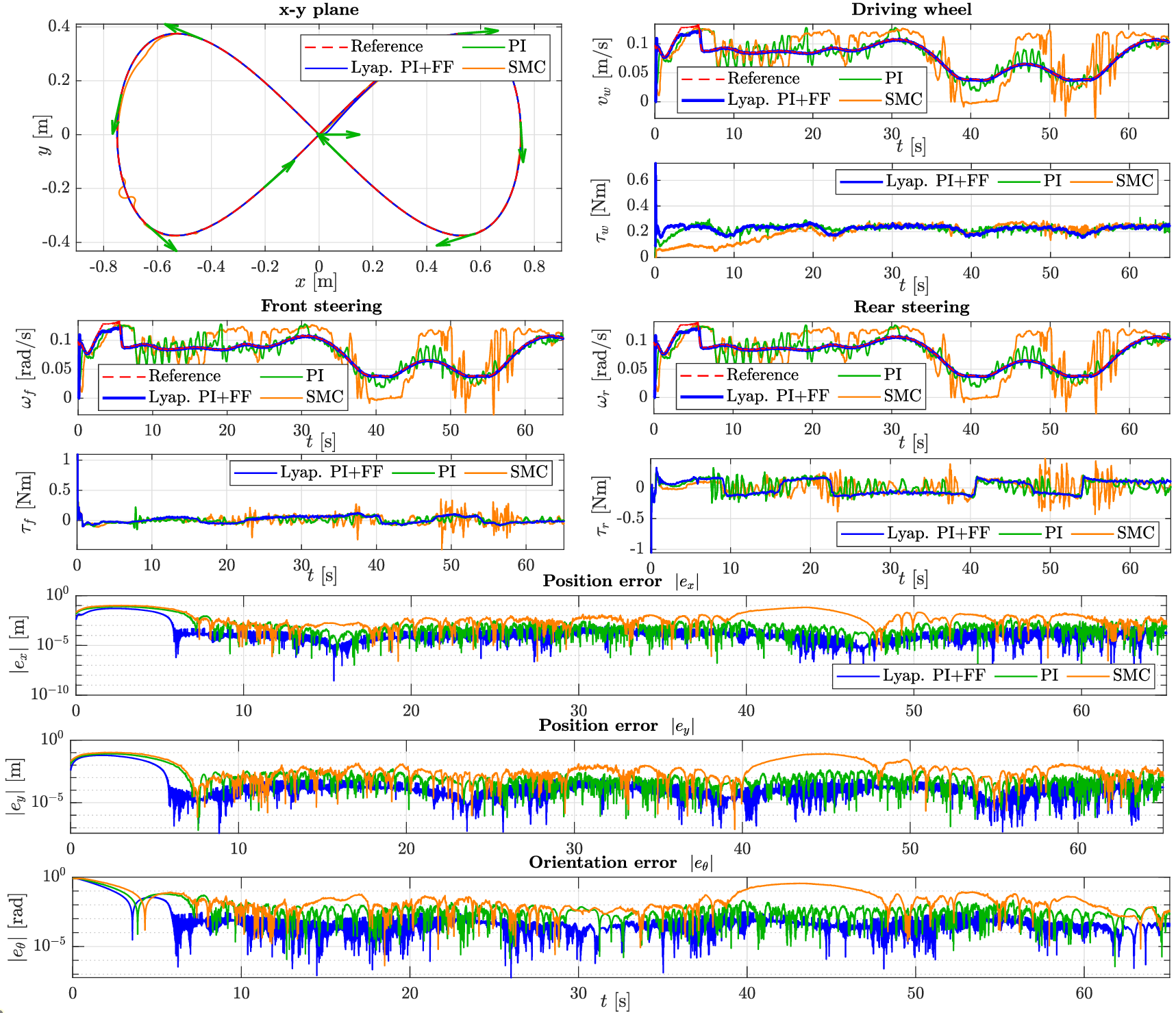}
\caption{Experimental responses of the robot subjected to a normal and axial disturbance, compared across the three different controllers. Green arrows indicate robot orientation.}\label{fig_comp_normal_axial_dist}
\end{figure*}

The baseline PI controller was implemented in the standard form $\bm{\tau}_{PI} = \mathbf{\tilde{B}}^{-1} \left( -\mathbf{K}_{P,PI}\mathbf{e}_v - \mathbf{K}_{I,PI} \bm{\eta} \right)$, $\bm{\dot{\eta}}=\mathbf{e}_v$. Here, $\mathbf{\tilde{B}}$, $\mathbf{e}_v$, and $\bm{\eta}$ are defined as in the proposed controller formulation. The implemented proportional and integral gains were chosen as diagonal matrices, $\mathbf{K}_{P,PI}=0.03\,\mathbf{I}_3$, $\mathbf{K}_{I,PI}=0.10\,\mathbf{I}_3$.

The sliding-mode controller (SMC) was implemented as a baseline for the drive-wheel speed subsystem. The same symmetric reduced reference generator was used as for the proposed controller, so that the common reduced drive velocity $v_{w,d}$ was mapped to the corresponding four-wheel drive-speed reference vector $\mathbf{v}_{rc}$. Let $\mathbf{e}=\mathbf{v}_{rc}-\mathbf{v}_r$ denote the drive-wheel speed tracking error, where $\mathbf{v}_r$ is the vector of measured wheel ground speeds used by the baseline drive-wheel subsystem. The SMC torque command is given by $\bm{\tau}_{SMC} = \mathbf{B}_r^{-1} \left[ \mathbf{\dot{v}}_{rc} + \bm{\Lambda}\mathbf{e} + \mathbf{A}_r\mathbf{v}_r + \mathbf{K}\operatorname{sat} \left( \mathbf{S}/\epsilon \right) \right]$, where $\mathbf{A}_r$ and $\mathbf{B}_r$ are obtained from the inertia, Coriolis, and input-map decomposition of the drive-wheel dynamic model. The matrix $\mathbf{A}_r$ represents the model-based Coriolis-like feedback term, $\mathbf{B}_r$ is the corresponding input matrix, and $\mathbf{K}=\operatorname{diag}(k_i)$ is the switching-gain matrix. The sliding surface is of PI-type and is defined as $\mathbf{S} = \mathbf{e} + \bm{\Lambda} \int_0^t \mathbf{e}(\xi) d\xi$. The function $\operatorname{sat}(\cdot)$ is applied component-wise within a boundary layer of half-width $\epsilon$ in order to reduce chattering.

The implemented SMC parameters were $\lambda_i=2.0~\mathrm{s}^{-1}$, $k_i=50.0~\mathrm{m/s^2}$, $\epsilon=1.0~\mathrm{m/s}$, $\tau_{\max}=1.0~\mathrm{Nm}$, $\rho=10^{-6}~\mathrm{m}^{-1}$.
Here, $\lambda_i$ is the integral gain in the PI-type sliding surface, $k_i$ is the switching gain, $\epsilon$ is the boundary-layer half-width, $\mu_i$ is the derivative gain, $\tau_{\max}$ is the drive-torque saturation limit, $\rho$ is the Tikhonov regularization parameter used in the input map, and $f_s$ is the control update frequency.

In the planar case \textit{ii)}, robustness is assessed by applying external disturbances via thrusters acting against the driving direction, while the adhesion system increases the effective normal load. In the vertical case \textit{iii)}, adhesion ensures wall contact and thrusters provide lifting force for gravity compensation, leading to significant disturbances and complex contact effects. The same controller gains are used in both experiments, highlighting that the selected gains satisfy the analytical stability condition while maintaining robust performance across substantially different disturbance and contact regimes.

Figs.~\ref{fig_comp_no_dist}, \ref{fig_comp_normal_dist}, and \ref{fig_comp_normal_axial_dist} present the results of the first experiment, in which the robot tracks a Lissajous trajectory defined by $x_d(t) = 0.75 \cos(0.1t - \pi/2 + 0.75)$, $y_d(t) = -0.5 \sin(0.2t - \pi)$, and $\theta_d(t) = \operatorname{atan2}(\dot{y}_d, \dot{x}_d)$ on a horizontal surface, using three controllers: the proposed controller, a PI controller, and a sliding-mode controller. The proposed controller achieves lower tracking error in the \textit{x-y} plane, together with smoother velocity responses and reduced actuation effort. This tracking performance is consistent with the nominal stability guarantees established by the proposed analysis under all three cases - no disturbance, normal disturbance and normal-axial disturbance. The results also show that the proposed controller is more robust to disturbances than the PI and sliding-mode controllers, which exhibit significant tracking errors and increased actuation effort in the presence of disturbances (see Table~\ref{tab:results}). The observed performance under disturbances is consistent with the practical stability and bounded error behaviour predicted by the proposed analysis. It should be emphasized that the reference curves shown in the plots correspond exclusively to the proposed controller case in order to prevent visual ambiguity.

Fig.~\ref{fig_flower_floor} presents results for a flower-shaped trajectory defined by $x_d(t) = 0.5 \cos(2\pi t/35) \cos(2\pi t/70) + 0.1$, $y_d(t) = 0.5 \cos(2\pi t/35) \sin(2\pi t/70) + 0.2$, and $\theta_d(t) = \operatorname{atan2}(\dot{y}_d, \dot{x}_d)$. This trajectory is more demanding than the Lissajous case, featuring a larger number of turns and more frequent changes in curvature, and is therefore used to evaluate the proposed controller under more challenging tracking conditions. To assess robustness, the figure overlays the response of the proposed controller under the different disturbance conditions in a single plot. The results show that the motor torques increase in the presence of the applied disturbances, as reflected in the total-variation metrics reported in Table~\ref{tab:results}; despite this transient increase in actuation effort, the velocity tracking performance remains stable despite this transient increase in actuation effort, the velocity tracking performance remains stable, which in turn ensures stable tracking of the kinematic trajectory in the \textit{x-y} plane. The observed behaviour is consistent with the practical stability and bounded-error guarantees established by the proposed analysis. It should be emphasized that the reference curves shown in the plots correspond exclusively to the nominal, disturbance-free trajectory in order to prevent visual ambiguity. 

The results of the third experiment, in which the robot moves on a vertical surface, are shown in Fig. \ref{fig_eight_wall}. The trajectory used for this experiment is defined as $x_d(t) = 0.75 \cos(0.1t - \pi/2 + 0.75)$, $y_d(t) = -0.5 \sin(0.2t - \pi)$, and $\theta_d(t) = 0.0$. It should be noted that, in order for the thruster motors to effectively contribute to assisted motion along the vertical surface, it is necessary to maintain a constant orientation of the robot, $\theta_d = 0.0$ rad. In contrast to the planar case, the motor torques remain relatively low, since the adhesion system ensures stable contact with the surface, while the thrusters assist the driven wheels in counteracting gravitational effects during upward motion. This is also reflected in the comparatively low total-variation values, denoted by $\mathrm{TV}_{(\cdot)}$, reported in Table~\ref{tab:results}. Furthermore, it can be observed that during the second part of the trajectory, when the robot descends along the vertical surface, the motor torques are close to zero, as the actuators primarily operate in a braking regime to maintain accurate trajectory tracking. In both motion phases, the tracking of the reference velocity trajectory, as well as the tracking of the kinematic trajectory in the \textit{z-y} plane, remains stable and accurate. Overall, the experiments indicate attenuation of gravity-related and contact-dependent effects without loss of stability, in agreement with the practical robustness and uniform ultimate boundedness interpretation of the proposed analysis.

\begin{figure*}[!t]
\centering
\includegraphics[scale=0.43]{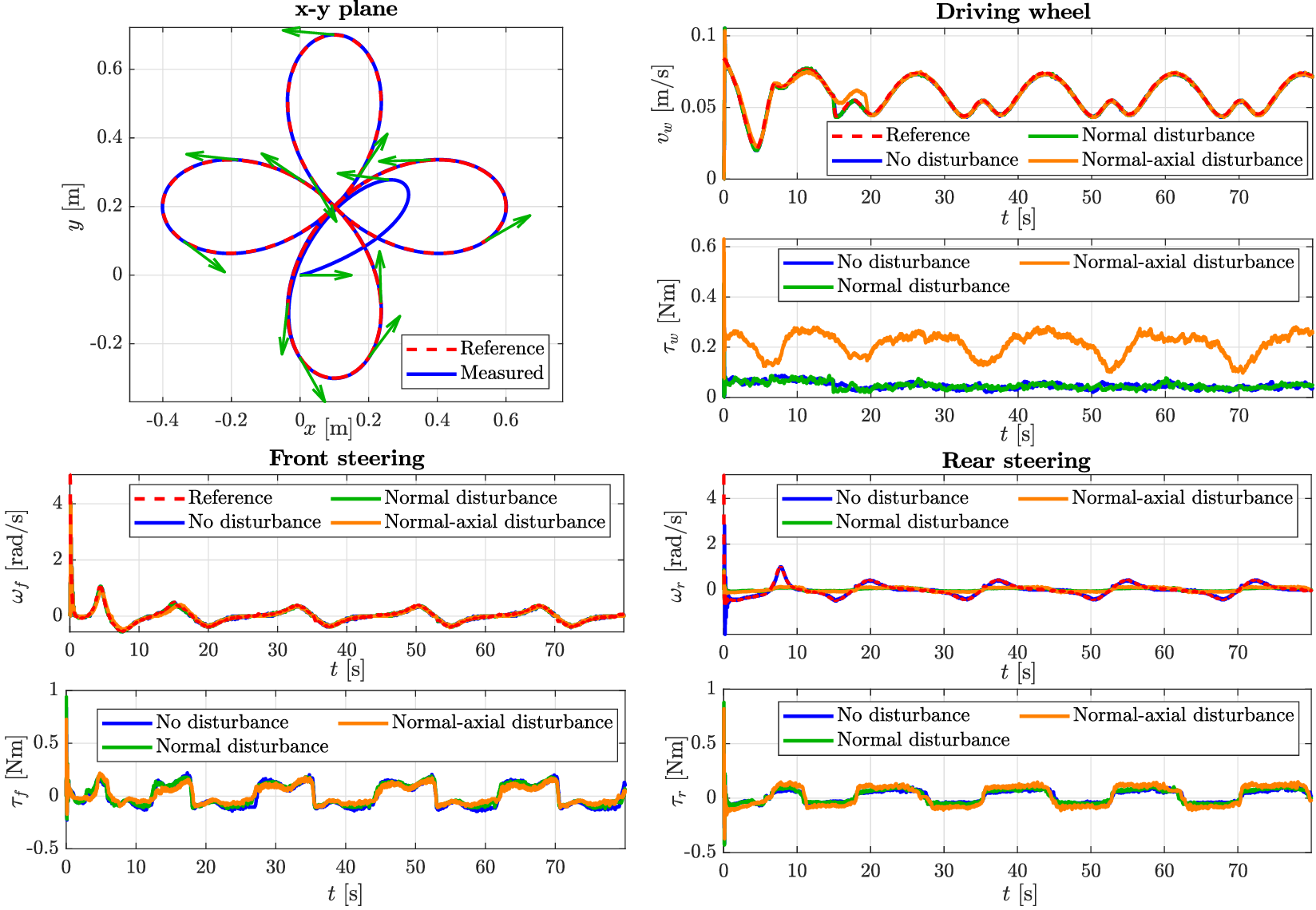}
\caption{Trajectory tracking with proposed Lyap. PI + FF on horizontal plane. Green arrows indicate robot orientation.}\label{fig_flower_floor}
\end{figure*}
\begin{figure*}[!t]
\centering
\includegraphics[scale=0.43]{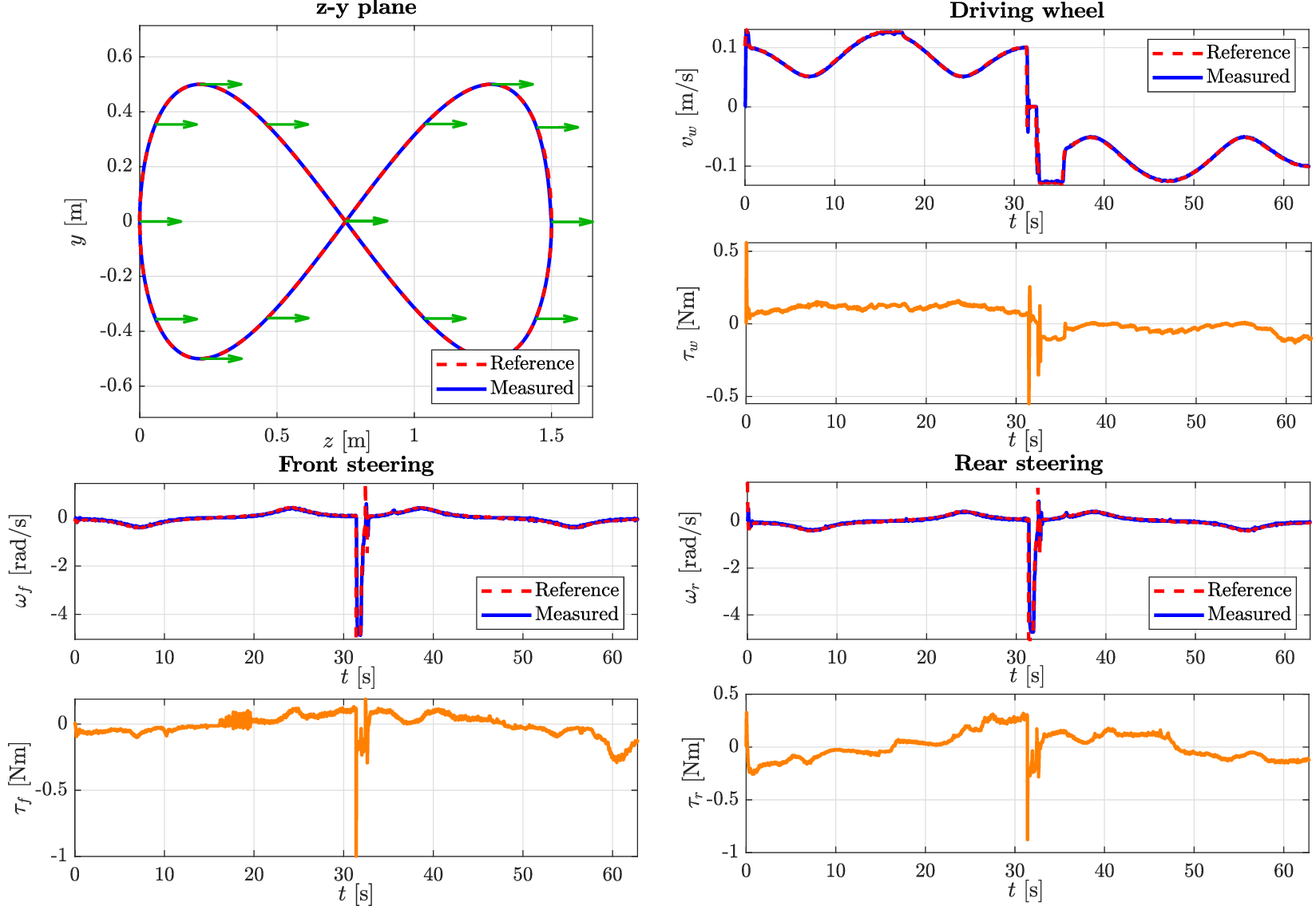}
\caption{Trajectory tracking with proposed Lyap. PI + FF on vertical surface. Green arrows indicate robot orientation.}\label{fig_eight_wall}
\end{figure*}

Short peaks in $\omega_{(\cdot)}$ are observed and are attributed to reference-generation transients of the virtual kinematic controller, i.e., abrupt changes in the mapped steering reference $\delta_{(\cdot)}$ at initialization and during mode transitions due to steering-angle unwrapping and sign changes in the commanded longitudinal velocity (e.g., at the top of the Lissajous trajectory). Hence, these peaks are regarded as brief reference-mapping events rather than sustained steering-rate demands, and $\dot{\delta}_{\max}=\pi/2\,\mathrm{rad/s}$ is adopted as a conservative bound for steady tracking.

A video demonstrating the conducted experiments is provided in the supplementary material.

\section{Conclusion}\label{sec_Conclusion}

This paper addressed the problem of robust trajectory tracking for an independently steered and driven four-wheel mobile robot through the development of a PI-like velocity control framework supported by a Lyapunov-based practical stability analysis of the closed-loop error dynamics. The proposed design combines nominal model compensation with proportional-integral feedback in the velocity space, which allows the controller to retain a simple implementation-oriented structure while explicitly accounting for residual dynamics and bounded unmodelled effects. A key outcome of the analysis is the construction of an augmented Lyapunov function for the combined velocity-error and integral-error state. This construction leads to explicit gain conditions under which the inner-loop velocity-error dynamics are practically stable and the combined inner-loop velocity-error and integral-error state is uniformly ultimately bounded, thereby providing robustness guarantees for the velocity-tracking loop without increasing the complexity of the resulting control law.

The approach was experimentally validated on a robot under diverse operating conditions, including horizontal motion with disturbances and vertical motion with gravity, adhesion, and complex contact effects, and was benchmarked against PI and sliding-mode controllers. The results demonstrate stable closed-loop behaviour and accurate trajectory tracking, with the proposed controller achieving lower tracking error and reduced actuation effort than both baselines, confirming its practical effectiveness. On the Lissajous benchmark, the proposed controller reduced translational tracking error by approximately 19\% relative to the PI baseline and 57\% relative to SMC, while lowering the total variation of the commanded inputs by roughly 45\% and 68\%, respectively. These reductions were most pronounced under the demanding normal–axial disturbance and indicate that the feedforward term improves accuracy without incurring additional actuation effort.

\bibliographystyle{IEEEtran}
\bibliography{references}

\vfill

\end{document}

%% file: Figures/control_scheme_kin_dyn_ver3.tex
\begin{tikzpicture}[x=1pt,y=1pt, line cap=round, line join=round, line width=0.4pt, >={Triangle[length=3.5mm,width=2.6mm,fill=black]}]

\node[rectangle, draw, line width=2.0pt, line join=miter, fill=white, minimum width=90.0pt, minimum height=34.0pt, inner sep=0pt, align=center] (n1) at (20.00,104.00) {Reference pose\\ $\left[x_d(t)\; y_d(t)\; \theta_d(t)\right]^{\text{T}}$};

\node[ellipse, draw, fill=white, minimum width=20.0pt, minimum height=20.0pt, inner sep=0pt, align=center] (n2) at (100.00,104.00) {};

\node[rectangle, draw, line width=2.0pt, line join=miter, fill=white, minimum width=90.0pt, minimum height=44.0pt, inner sep=0pt, align=center] (n3) at (215.50,104.00) {\shortstack{Virtual kinematic\\ controller}};

\node[rectangle, draw, line width=2.0pt, line join=miter, fill=white, minimum width=120.0pt, minimum height=35.0pt, inner sep=0pt, align=center] (n4) at (435.00,104.00) {$\mathbf{\tilde{B}}^{-1}(\mathbf{q}) \Big( -{\mathbf{K}}_P \mathbf{e}_v - \mathbf{K}_I \bm{\eta} \Big)$\\ $\bm{\dot{\eta}} = \mathbf{e}_v$};

\node[ellipse, draw, fill=white, minimum width=20.0pt, minimum height=20.0pt, inner sep=0pt, align=center] (n5) at (530.00,104.00) {};

\node[rectangle, draw, line width=2.0pt, line join=miter, fill=white, minimum width=63.0pt, minimum height=50.0pt, inner sep=0pt, align=center] (n6) at (600.00,104.00) {\shortstack{Robot\\ dynamics}};

\node[rectangle, draw, line width=2.0pt, line join=miter, fill=white, minimum width=160.0pt, minimum height=25.0pt, inner sep=0pt, align=center] (n7) at (430.00,145.00) {$\mathbf{\tilde{B}}^{-1}(\mathbf{q}) \Big( \mathbf{\tilde{M}}(\mathbf{q}_d)\mathbf{\dot{v}}_d + \mathbf{\tilde{C}}(\mathbf{q}_d, \mathbf{\dot{q}}_d)\mathbf{v}_d \Big)$};

\node[rectangle, draw, line width=2.0pt, line join=miter, fill=white, minimum width=30.0pt, minimum height=24.0pt, inner sep=0pt, align=center] (n8) at (280.00,52.00) {$\mathbf{J}(\mathbf{q})$};

\node[rectangle, draw, line width=2.0pt, line join=miter, fill=white, minimum width=20.0pt, minimum height=24.0pt, inner sep=0pt, align=center] (n9) at (220.00,52.00) {$\int$};

\node[rectangle, draw, line width=2.0pt, line join=miter, fill=white, minimum width=55.0pt, minimum height=24.0pt, inner sep=1pt, align=center] (n11) at (140.00,52.00) {$\left[\mathbf{I}_{3 \times 3}\; \mathbf{0}_{3 \times 3}\right]$};

\node[ellipse, draw, fill=white, minimum width=20.0pt, minimum height=20.0pt, inner sep=0pt, align=center] (n10) at (340.00,104.00) {};
\node[draw=none, inner sep=0pt, align=center] at (138.00,114.00) {$\left[e_x\; e_y\; e_{\theta}\right]^{\text{T}}$};
\node[draw=none, inner sep=0pt, align=center] at (303.50,114.00) {$\mathbf{v}_d$};
\node[draw=none, inner sep=0pt, align=center] at (550.00,111.67) {$\bm{\tau}$};
\node[draw=none, inner sep=0pt, align=center] at (90.00,85.00) {$-$};
\node[draw=none, inner sep=0pt, align=center] at (80.00,114.00) {$+$};
\node[draw=none, inner sep=0pt, align=center] at (520.00,128.00) {$+$};
\node[draw=none, inner sep=0pt, align=center] at (510.00,116.00) {$+$};
\node[draw=none, inner sep=0pt, align=center] at (459.00,49.00) {$\mathbf{v}$};
\node[draw=none, inner sep=0pt, align=center] at (315.50,110.00) {$-$};
\node[draw=none, inner sep=0pt, align=center] at (330.50,85.00) {$+$};
\node[draw=none, inner sep=0pt, align=center] at (361.50,113.50) {$\mathbf{e}_v$};
\node[draw=none, inner sep=0pt, align=center] at (200.00,59.00) {$\mathbf{q}$};
\node[draw=none, inner sep=0pt, align=center] at (610.00,159.00) {$\tilde{\mathbf{f}}$};

\draw[->] (n1.east) -- (n2.west);

\draw[->] (n2.east) -- (n3.west);

\draw[->] (n3.east) -- (287.00,104.00) |- ([yshift=-10pt]n7.west);
\fill[black] (287.00,104.00) circle (2pt);

\draw[->] (n5.east) -- (n6.west);

\draw[->] (n4.east) -- (n5.west);

\draw[->] (n6.east) -- (650.00,104.00) -- (650.00,45.00) -- ([yshift=-7pt]n8.east);

\draw[->] (n7.east) -| (n5.north);

\draw[->] (n8.west) -- (n9.east);

\draw[->] (n9.west) -- (n11.east) coordinate[pos=0.5] (cMidN9N11);
\fill (cMidN9N11) circle (2pt);

\coordinate (n4Sdown) at ([yshift=-15pt]n4.south);
\draw[->] (cMidN9N11)
-- (cMidN9N11 |- n4Sdown)
-- (n4.south |- n4Sdown)
-- (n4.south);

\coordinate (cToN7) at ([xshift=10pt]n3.east |- n4Sdown);

\fill (cToN7) circle (2pt);

\draw[->] (cToN7) |- (n7.west);

\coordinate (cToN8) at ([xshift=65pt]n3.east |- n4Sdown);

\fill (cToN8) circle (2pt);

\draw[->] (cToN8) |- ([yshift=7pt]n8.east);

\draw[->] (n11.west) -| (n2.south);

\draw[->] (n10.east) -- (n4.west);

\draw[<-] (329.00,104.00) -- (273.00,104.00);

\draw[<-] (n10.south) -- (340.00,45.00);
\fill[black] (340.00,45.00) circle (2pt);

\draw[->] ([yshift=40pt]n6.north) -- (n6.north);

\coordinate (n7Wup) at ([yshift=10pt]n7.west);
\coordinate (cN3N7) at (n3.north |- n7Wup);

\draw[->] (n3.north) -- (cN3N7)
-- node[midway, above, align=center]
{$\delta_{f,d},\; \delta_{r,d},\; \dot{\delta}_{f,d},\; \dot{\delta}_{r,d}$} (n7Wup);

\end{tikzpicture}